\documentclass[11pt]{article}

\usepackage{times}
\usepackage{macros}
\usepackage{fullpage}
\usepackage{amsfonts,graphicx}
\usepackage{amsmath,amssymb}
\usepackage{graphicx}
\usepackage{color}
\usepackage{enumerate}
\usepackage[numbers]{natbib}
\usepackage{algorithm}
\usepackage[bookmarks=true,colorlinks,citecolor=blue,urlcolor=blue]{hyperref}
\usepackage{amsthm}

\newtheorem{theorem}{Theorem}
\newtheorem{claim}{Claim}
\newtheorem{lemma}{Lemma}
\newtheorem{corollary}{Corollary}

\newtheorem*{definition*}{Definition}
\newtheorem{proposition}{Proposition}
\newtheorem{assumption}{Assumption}

\begin{document}

\title{A Hitting Time Analysis of Stochastic Gradient\\ Langevin Dynamics}

\author{
Yuchen Zhang\footnote{\footnotesize Computer Science Department, Stanford University, Stanford, CA 94305. Email: {\tt zhangyuc@cs.stanford.edu}.}
\qquad
Percy Liang\footnote{\footnotesize Computer Science Department, Stanford University, Stanford, CA 94305. Email: {\tt pliang@cs.stanford.edu}.}
\qquad
Moses Charikar\footnote{\footnotesize Computer Science Department, Stanford University, Stanford, CA 94305. Email: {\tt moses@cs.stanford.edu}.}
}

\maketitle

\begin{abstract}
\noindent
We study the Stochastic Gradient Langevin Dynamics (SGLD) algorithm for non-convex optimization. The algorithm performs stochastic gradient descent, where in each step it injects appropriately scaled Gaussian noise to the update. We analyze the algorithm's hitting time to an arbitrary subset of the parameter space. Two results follow from our general theory: First, we prove that for empirical risk minimization, if the empirical risk is pointwise close to the (smooth) population risk, then the algorithm finds an approximate local minimum of the population risk in polynomial time, escaping suboptimal local minima that only exist in the empirical risk. Second, we show that SGLD improves on one of the best known learnability results for learning linear classifiers under the zero-one loss.
\end{abstract}


\section{Introduction}

A central challenge of non-convex optimization is avoiding sub-optimal local minima. Although escaping all local minima is NP-hard in general~\citep[e.g.][]{blum1992training}, one might expect that it should be possible to escape ``appropriately shallow'' local minima, whose basins of attraction have relatively low barriers. As an illustrative example, consider minimizing an empirical risk function in Figure~\ref{fig:local-min}. As the figure shows, although the empirical risk is uniformly close to the population risk, it contains many poor local minima that don't exist in the population risk. Gradient descent is unable to escape such local minima.

A natural workaround is to inject random noise to the gradient. Empirically, adding gradient noise has been found to improve learning for deep neural networks and other non-convex models~\citep{neelakantan2015neural,neelakantan2015adding,kurach2015neural,
kaiser2015neural,zeyer2016comprehensive}.
However, theoretical understanding of the value of gradient noise is still incomplete.
For example, \citet{ge2015escaping} show that by adding isotropic noise $w$ and by choosing a sufficiently small stepsize $\eta$, the iterative update:
\begin{align}\label{eqn:ge-sgd}
x \leftarrow x - \eta\,(\nabla f(x) + w)
\end{align}
is able to escape strict saddle points. Unfortunately, this approach, as well as the subsequent line of work on escaping saddle points~\citep{lee2016gradient, anandkumar2016efficient, agarwal2016finding}, doesn't guarantee escaping even shallow local minima.

Another line of work in Bayesian statistics studies the Langevin Monte Carlo (LMC) method \citep{roberts1996exponential}, which employs an alternative noise term. Given a function $f$, LMC performs the iterative update:
\begin{align}\label{eqn:simplified-LMC}
x \leftarrow x - \eta\,(\nabla f(x) + \sqrt{2/(\eta \xi)}\,w)\quad \mbox{where}\quad w\sim N(0, I),
\end{align}
where $\xi>0$ is a ``temperature'' hyperparameter. Unlike the bounded noise added in formula~\eqref{eqn:ge-sgd}, LMC adds a large noise term that scales with $\sqrt{1/\eta}$. With a small enough $\eta$, the noise dominates the gradient, enabling the algorithm to escape any local minimum. For empirical risk minimization, one might substitute the exact gradient $\nabla f(x)$ with a stochastic gradient, which gives the Stochastic Gradient Langevin Dynamics (SGLD) algorithm~\citep{welling2011bayesian}. It can be shown that both LMC and SGLD asymptotically converge to a stationary distribution $\mu(x)\propto e^{-\xi f(x)}$~\citep{roberts1996exponential, teh2016consistency}. As $\xi\to\infty$, the probability mass  of $\mu$ concentrates on the global minimum of the function~$f$, and the algorithm asymptotically converges to a neighborhood of the global minimum.

\begin{figure}
\centering
\includegraphics[width=0.45\textwidth]{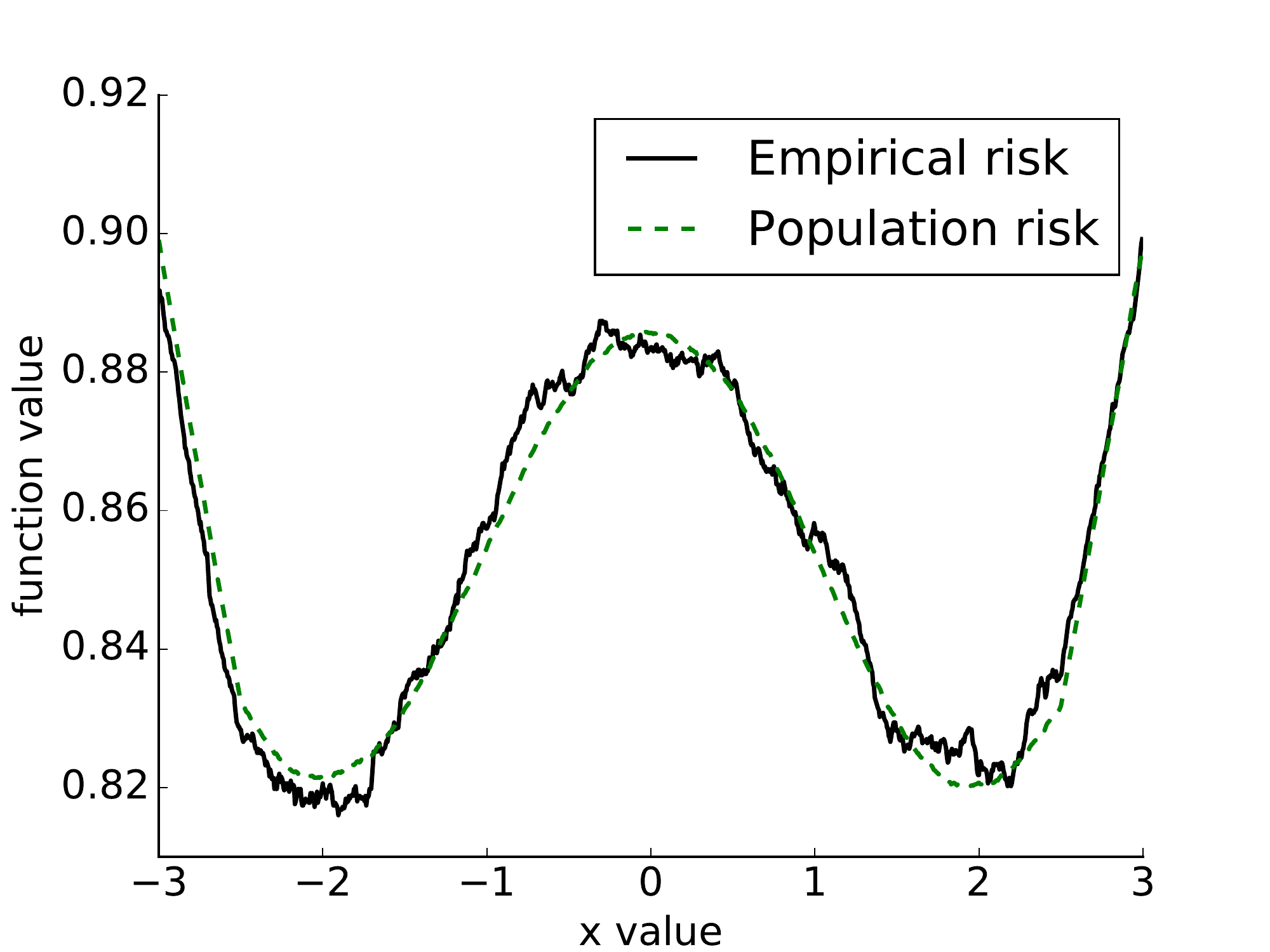}
\caption{\small Empirical risk (sample size $=5000$) versus population risk (sample size $\to\infty$) on one-dimensional zero-one losses. The two functions are uniformly close, but the empirical risk contains local minima that that are far worse than the population local minima.}
\label{fig:local-min}
\end{figure}

Despite asymptotic consistency, there is no theoretical guarantee that LMC is able to find the global minimum of a general non-convex function, or even a local minimum of it, in polynomial time. Recent works focus on bounding the mixing time (i.e.~the time for converging to $\mu$) of LMC and SGLD. \citet{bubeck2015finite}, \citet{dalalyan2016theoretical} and \citet{bonis2016guarantees} prove that on convex functions, LMC converges to the stationary distribution in polynomial time. On non-convex functions, however, an exponentially long mixing time is unavoidable in general. According to \citet{bovier2004metastability}, it takes the Langevin diffusion at least $e^{\Omega(\xi h)}$ time to escape a depth-$h$ basin of attraction. Thus, if the function contains multiple ``deep'' basins with $h=\Omega(1)$, then the mixing time is lower bounded by~$e^{\Omega(\xi)}$. 

In parallel work to this paper, \citet{raginsky2017nonconvex} upper bound the time of SGLD converging to an approximate global minimum of non-convex functions. They show that the upper bound is polynomial in the inverse of a quantity they call the \emph{uniform spectral gap}. Similar to the mixing time bound, in the presence of multiple local minima, the convergence time to an approximate global minimum can be exponential in dimension~$d$ and the temperature parameter~$\xi$.

\paragraph{Contributions} In this paper, we present an alternative analysis of SGLD algorithm.\footnote{The theory holds for the standard LMC algorithm as well.} Instead of bounding its mixing time, we bound the algorithm's \emph{hitting time} to an arbitrary set $U$ on a general non-convex function. The hitting time captures the algorithm's optimization efficiency, and more importantly, it enjoys polynomial rates for hitting appropriately chosen sets regardless of the mixing time, which could be exponential. We highlight two consequences of the generic bound: First, under suitable conditions, SGLD hits an approximate local minimum of $f$, with a hitting time that is polynomial in dimension $d$ and all hyperparameters; this extends the polynomial-time guarantees proved for convex functions~\citep{bubeck2015finite, dalalyan2016theoretical, bonis2016guarantees}.
Second, the time complexity bound is stable, in the sense that any $\order(1/\xi)$ perturbation in $\ell_\infty$-norm of the function~$f$ doesn't significantly change the hitting time. This second property is the main strength of SGLD: For any function $f$, if there exists another function $F$ such that $\norms{f-F}_\infty = \order(1/\xi)$, then we define the set $U$ to be the approximate local minima of $F$. The two properties together imply that even if we execute SGLD on function $f$, it hits an approximate local minimum of $F$ in polynomial time. In other words, SGLD is able to escape ``shallow'' local minima of $f$ that can be eliminated by slightly perturbing the function.

This stability property is useful in studying empirical risk minimization (ERM) in situations where the empirical risk $f$ is pointwise close to the population risk $F$, but has poor local minima that don't exist in the population risk. This phenomenon has been observed in statistical estimation with non-convex penalty functions~\citep{wang2014optimal, loh2015regularized}, as well as in minimizing the zero-one loss (see Figure~\ref{fig:local-min}). Under this setting, our result implies that SGLD achieves an approximate local minimum of the (smooth) \emph{population risk} in polynomial time, ruling out local minima that only exist in the empirical risk. It improves over recent results on non-convex optimization~\citep{ge2015escaping,lee2016gradient, anandkumar2016efficient, agarwal2016finding}, which compute approximate local minima only for the empirical risk.

As a concrete application, we prove a stronger learnability result for the problem of learning linear classifiers under the zero-one loss~\citep{arora1993hardness}, which involves non-convex and non-smooth empirical risk minimization.
Our result improves over the recent result of \citet{awasthi2015efficient}: the method of \citet{awasthi2015efficient} handles noisy data corrupted by a very small Massart noise (at most $1.8\times 10^{-6}$), while our algorithm handles Massart noise up to any constant less than $0.5$.
As a Massart noise of $0.5$ represents completely random observations,
we see that SGLD is capable of learning from very noisy data.

\paragraph{Techniques} The key step of our analysis is to define a positive quantity called the \emph{restricted Cheeger constant}. This quantity connects the hitting time of SGLD, the geometric properties of the objective function, and the stability of the time complexity bound. For an arbitrary function $f:K\to \R$ and an arbitrary set $V\subset K$, the restricted Cheeger constant is defined as the minimal ratio between the surface area of a subset $A\subset V$ and its volume with respect to a measure $\mu(x)\propto e^{-f(x)}$. We prove that the hitting time is polynomial in the inverse of the restricted Cheeger constant (Section~\ref{sec:generic-bounds}). The stability of the time complexity bound follows as a natural consequence of the definition of this quantity (Section~\ref{sec:define-rcc}). We then develop techniques to lower bound the restricted Cheeger constant based on geometric properties of the objective function (Section~\ref{sec:cheeger-lower-bounds}). 

\paragraph{Notation} For any positive integer $n$, we use $[n]$ as a
shorthand for the discrete set $\{1,2,\dots, n\}$.  For a rectangular
matrix $A$, let $\lncs{A}$ be its nuclear norm (i.e., the sum of singular values), and $\ltwos{A}$ be its spectral norm (i.e., the maximal singular value). For any point $x\in \R^d$ and an arbitrary set $V\subset \R^d$, we denote their Euclidean distance by $d(x,V)\defeq \inf_{y\in V} \ltwos{x-y}$. We use $\ball(x;r)$ to denote the Euclidean ball of radius $r$ that centers at point $x$.

\section{Algorithm and main results}
\label{sec:main-theorem}

In this section, we define the algorithm and the basic concepts, then present the main theoretical results of this paper.

\subsection{The SGLD algorithm}

\begin{algorithm}[t]
\small
\begin{flushleft}
{\textbf Input: }{Objective function $f:K \to \R$; hyperparameters $(\xi,\eta,\kmax,\DistBound)$.}
\begin{enumerate}
\item Initialize $x_0\in K$ by uniformly sampling from the parameter space $K$.
\item For each $k \in \{1,2,\dots,\kmax\}$: Sample $w\sim N(0, I_{d\times d})$. Compute a stochastic gradient $g(x_{k-1})$ such that $\E[g(x_{k-1})|x_{k-1}] = \nabla f(x_{k-1})$. Then update:
\begin{subequations}
\begin{align}
&y_k = x_{k-1} - \eta\, g(x_{k-1}) + \sqrt{2\eta/\xi}\,w;\label{eqn:sld-update-1}\\
&x_k = \left\{\begin{array}{ll}
	y_k & \mbox{if } y_k\in K\cap \ball(x_{k-1};\DistBound),\\
	x_{k-1} & \mbox{otherwise}.
\end{array}\right.	
\label{eqn:sld-update-2}
\end{align}
\end{subequations}
\end{enumerate}
\vspace{-10pt}
  {\textbf Output: }$\xhat = x_{k^*} \text{ where }  k^* \defeq \argmin_k \{f(x_k)\}$.\vspace{-10pt}
\end{flushleft}
\caption{Stochastic Gradient Langevin Dynamics}\label{alg:SLD}
\end{algorithm}

Our goal is to minimize a function $f$ in a compact parameter space $K\subset \R^d$. The SGLD algorithm~\citep{welling2011bayesian} is summarized in Algorithm~\ref{alg:SLD}. In step~\eqref{eqn:sld-update-1}, the algorithm performs SGD on the function~$f$, then adds Gaussian noise to the update. Step~\eqref{eqn:sld-update-2} ensures that the vector $x_k$ always belong to the parameter space, and is not too far from $x_{k-1}$ of the previous iteration.\footnote{The hyperparameter $\DistBound$ can be chosen large enough so that the constraint $y_k\in \ball(x_{k-1};\DistBound)$ is satisfied with high probability, see Theorem~\ref{theorem:sld-general}.} After $\kmax$ iterations, the algorithm returns a vector~$\xhat$. 
Although standard SGLD returns the last iteration's output, we study a variant of the algorithm which returns the best vector across all iterations. This choice is important for our analysis of hitting time. We note that evaluating $f(x_k)$ can be computationally more expensive than computing the stochastic gradient $g_k$, because the objective function is defined on the entire dataset, while the stochastic gradient can be computed via a single instance. Returning the best $x_k$ merely facilitates theoretical analysis and might not be necessary in practice.

Because of the noisy update, the sequence $(x_0,x_1,x_2,\dots)$ asymptotically converges to a stationary distribution rather than a stationary point~\citep{teh2016consistency}. Although this fact introduces challenges to the analysis,
we show that its non-asymptotic efficiency can be characterized by a positive quantity called the \emph{restricted Cheeger constant}.

\subsection{Restricted Cheeger constant}
\label{sec:define-rcc}

For any measurable function $f$, we define a probability measure $\mu_f$ whose density function is:
\begin{align}\label{eqn:define-prob-measure}
	\mu_f(x) \defeq \frac{e^{-f(x)}}{\int_K e^{-f(x)}dx} \propto e^{-f(x)} \quad\mbox{for all}\quad x\in K.
\end{align}
For any function $f$ and any subset $V\subset K$, we define the restricted Cheeger constant as:
\begin{align}\label{eqn:cheeger-constant}
	\C_{f}(V) \defeq \liminf_{\SmallVar\searrow 0} \inf_{A\subset V}  \frac{\mu_f(A_\SmallVar)-\mu_f(A)}{\SmallVar\, \mu_f(A)},\quad\mbox{where}\quad A_\SmallVar \defeq \{x\in K: d(x,A)\leq \SmallVar \}.
\end{align}
The restricted Cheeger constant generalizes the notion of the Cheeger isoperimetric constant~\citep{cheeger1969lower}, quantifying how well a subset of $V$ can be made as least connected as possible to the rest of the parameter space. The connectivity is measured by the ratio of the surface measure $\liminf_{\SmallVar\searrow 0} \frac{\mu_f(A_\SmallVar)-\mu_f(A)}{\SmallVar}$ to the set measure $\mu_f(A)$.
Intuitively, this quantifies the chance of escaping the set $A$ under the probability measure $\mu_f$.

\paragraph{Stability of restricted Cheeger constant} A property that will be important in the sequal is that the restricted Cheeger constant is stable under perturbations:
if we perturb $f$ by a small amount, then the values of $\mu_f$ won't change much, so that the variation on $\C_{f}(V)$ will also be small. More precisely, for functions $f_1$ and $f_2$ satisfying $\sup_{x\in K} |f_1(x) - f_2(x)| = \nu$, we have
\begin{align}\label{eqn:robustness-property}
	\C_{f_1}(V) = \liminf_{\SmallVar\searrow 0} \inf_{A\subset V}  \frac{\int_{A_\SmallVar\backslash A} e^{-f_1(x)}dx}{\SmallVar\int_A e^{-f_1(x)}dx}
	\geq \liminf_{\SmallVar\searrow 0} \inf_{A\subset V}  \frac{\int_{A_\SmallVar\backslash A} e^{-f_2(x)-\nu}dx}{\SmallVar\int_A e^{-f_2(x)+\nu}dx}
	= e^{-2\nu}\C_{f_2}(V),
\end{align}
and similarly $\C_{f_2}(V) \geq e^{-2\nu}\C_{f_1}(V)$. As a result, if two functions $f_1$ and $f_2$ are uniformly close, then we have $\C_{f_1}(V) \approx \C_{f_2}(V)$ for a constant $\nu$. This property enables us to lower bound $\C_{f_1}(V)$ by lower bounding the restricted Cheeger constant of an alternative function $f_2 \approx f_1$,
which might be easier to analyze.
%

\subsection{Generic non-asymptotic bounds}
\label{sec:generic-bounds}

We make several assumptions on the parameter space and on the objective function.

\begin{assumption}[parameter space and objective function]\label{assumption:conditions-for-szld}
~
\begin{itemize}
\item The parameter space $K$ satisfies: there exists $\hmax > 0$, such that for any $x\in K$ and any $h\leq \hmax$, the random variable $y\sim N(x, 2h I)$ satisfies $P(y\in K) \geq \frac{1}{3}$.


\item The function $f: K\to [0,B]$ is bounded, differentiable and $\Lf$-smooth in $K$, meaning that for any $x,y\in K$, we have $|f(y)-f(x)-\inprod{y-x}{\nabla f(x)}| \leq \frac{\Lf}{2}\ltwos{y-x}^2$.
\item The stochastic gradient vector $g(x)$ has sub-exponential tails: there exists $\bmax > 0$, $\Gg > 0$, such that given any $x\in K$ and any vector $u\in \R^d$ satisfying $\ltwos{u} \leq \bmax$, the vector $g(x)$ satisfies $\E\left[e^{\langle u, g(x)\rangle^2} \mid x \right] \leq \exp(\Gg^2 \ltwos{u}^2)$.
\end{itemize}
\end{assumption}

The first assumption states that the parameter space doesn't contain sharp corners, so that the update~\eqref{eqn:sld-update-2} won't be stuck at the same point for too many iterations. It can be satisfied, for example, by defining the parameter space to be an Euclidean ball and choosing $\hmax = o(d^{-2})$. The probability $1/3$ is arbitrary and can be replaced by any constant in $(0,1/2)$. The second assumption requires the function $f$ to be smooth.
We show how to handle non-smooth functions in Section~\ref{sec:erm} by appealing to the stability property of the restricted Cheeger constant discussed earlier.
The third assumption requires the stochastic gradient to have sub-exponential tails, which is a standard assumption in stochstic optimization.

\begin{theorem}\label{theorem:sld-general}
Assume that Assumption~\ref{assumption:conditions-for-szld} holds. For any subset $U\subset K$ and any $\xi,\rho,\SmallProbParam > 0$, there exist $\eta_0 > 0$ and $\kmax\in\mathbb{Z}^+$, such that if we choose any stepsize $\eta \in (0, \eta_0]$ and hyperparameter $\DistBound \defeq 4\sqrt{2\eta d/\xi}$, then with probability at least $1-\SmallProbParam$, SGLD after $k_{\max}$ iterations returns a solution $\xhat$ satisfying:
\begin{align}\label{eqn:szld-optimality-gap}
	f(\xhat) \leq \sup_{x:\,d(x,U)\leq \rho} f(x).
\end{align}
The iteration number $\kmax$ is bounded by
\begin{align}\label{eqn:kmax-bound}
	\kmax \leq \frac{M}{\min\{1,\C_{(\xi f)} (K\backslash U)\}^4}
\end{align}
where the numerator $M$ is polynomial in $(B, \Lf, \Gg, \log(1/\SmallProbParam), d, \xi, \eta_0/\eta, \hmax^{-1}, \bmax^{-1}, \rho^{-1})$. See Appendix~\ref{sec:subsection-for-the-main-theorem} for the explicit polynomial dependence.
\end{theorem}

Theorem~\ref{theorem:sld-general} is a generic result that applies to all optimization problems satisfying Assumption~\ref{assumption:conditions-for-szld}. The right-hand side of the bound~\eqref{eqn:szld-optimality-gap} is determined by the choice of $U$. If we choose $U$ to be the set of (approximate) local minima, and let $\rho > 0$ be sufficiently small, then $f(\xhat)$ will roughly be bounded by the worst local minimum.
The theorem permits $\xi$ to be arbitrary provided the stepsize $\eta$ is small enough. Choosing a larger $\xi$ means adding less noise to the SLGD update, which means that the algorithm will be more efficient at finding a stationary point, but less efficient at escaping local minima. Such a trade-off is captured by the restricted Cheeger constant in inequality~\eqref{eqn:kmax-bound} and will be rigorously studied in the next subsection.

The iteration complexity bound is governed by the restricted Cheeger constant. For any function~$f$ and any target set $U$ with a positive Borel measure, the
restricted Cheeger constant is strictly positive (see Appendix~\ref{sec:positive-cheeger}), so that with a small enough $\eta$, the algorithm always converges to the global minimum asymptotically. We remark that the SGD doesn't enjoy the same asymptotic optimality guarantee, because it uses a $O(\eta)$ gradient noise in contrast to SGLD's $O(\sqrt{\eta})$ one. Since the convergence theory requires a small enough $\eta$, we often have $\eta \ll \sqrt{\eta}$. the SGD noise is too conservative to allow the algorithm to escape local minima.

\begin{proofsketch} The proof of Theorem~\ref{theorem:sld-general} is fairly technical. We defer the full proof to Appendix~\ref{sec:proof-theorem-sld-general}, only sketching the basic proof ideas here. At a high level, we establish the theorem by bounding the hitting time of the Markov chain $(x_0, x_1,x_2,\dots)$ to the set $U_\rho \defeq\{x: d(x,U)\leq \rho\}$. Indeed, if some $x_k$ hits the set, then:
\[
	f(\xhat) \leq f(x_k) \leq \sup_{x\in U_\rho} f(x),
\]
which establishes the risk bound~\eqref{eqn:szld-optimality-gap}. 

In order to bound the hitting time, we construct a time-reversible Markov chain, and prove that its hitting time to $U_\rho$ is on a par with the original hitting time. To analyze this second Markov chain, we define a notion called the \emph{restricted conductance}, which measures how easily the Markov chain can transition between states within $K\backslash U_\rho$. This quantity is related to the notion of \emph{conductance} in the analysis of time-reversible Markov processes~\citep{lovasz1993random}, but the ratio between these two quantities can be exponentially large for non-convex $f$. We prove that the hitting time of the second Markov chain depends inversely on the restricted conductance, so that the problem reduces to lower bounding the restricted conductance.

Finally, we lower bound the restricted conductance by the restricted Cheeger constant. The former quantity characterizes the Markov chain, while the later captures the geometric properties of the function $f$. Thus, we must analyze the SGLD algorithm in depth to establish a connection between them. Once we prove this lower bound, putting all pieces together completes the proof. \end{proofsketch}


\subsection{Lower bounding the restricted Cheeger constant}
\label{sec:cheeger-lower-bounds}

In this subsection, we prove lower bounds on the restricted Cheeger constant $\C_{(\xi f)}(K\backslash U)$ in order to flesh out the iteration complexity bound of Theorem~\ref{theorem:sld-general}. We start with a lower bound for the class of convex functions:

\begin{proposition}
\label{prop:cheeger-for-convex-function}
Let $K$ be a $d$-dimensional unit ball. For any convex $\Gp$-Lipschitz continuous  function $f$ and any $\epsilon > 0$, let the set of $\epsilon$-optimal solutions be defined by:
\[	
	U \defeq \{x\in K:~ f(x) \leq \inf_{y\in K} f(y) + \epsilon\}. 
\]
Then for any $\xi \geq \frac{2 d \log(4\Gp/\epsilon)}{\epsilon}$, we have $\C_{(\xi f)}(K\backslash U) \geq 1$.
\end{proposition}

The proposition shows that if we choose a big enough $\xi$, then $\C_{(\xi f)}(K\backslash U)$ will be lower bounded by a universal constant. The lower bound is proved based on an isoperimetric inequality for log-concave distributions. See Appendix~\ref{sec:proof-prop-cheeger-for-convex-function} for the proof.

For non-convex functions, directly proving the lower bound is difficult, because the definition of $\C_{(\xi f)}(K\backslash U)$ involves verifying the properties of all subsets $A\subset K\backslash U$. We start with a generic lemma that reduces the problem to checking properties of all points in $K\backslash U$.

\begin{lemma}\label{lemma:vector-field}
Consider an arbitrary continuously differentiable vector field $\phi: K\to \R^d$ and a positive number $\SmallVar_0 >0$ such that:
\begin{align}\label{eqn:vector-field-conditions}
	\ltwos{\phi(x)} \leq 1 \quad \mbox{and} \quad x - \SmallVar\, \phi(x) \in K \quad \mbox{for any} \quad \SmallVar \in [0, \SmallVar_0],~ x\in K.
\end{align}
For any continuously differentiable function $f: K\to \R$ and any subset $V\subset K$, the restricted Cheeger constant $\C_{f}(V)$ is lower bounded by
\[
	\C_{f}(V) \geq \inf_{x\in V} \Big\{\inprod{\phi(x)}{\nabla f(x)}-\dvg \phi(x) \Big\} \quad \mbox{where} \quad \dvg \phi(x) \defeq \sum_{i=1}^d \frac{\partial \phi_i(x)}{\partial x_i}.
\]
\end{lemma}

\begin{figure}
\centering
\includegraphics[width=0.35\textwidth]{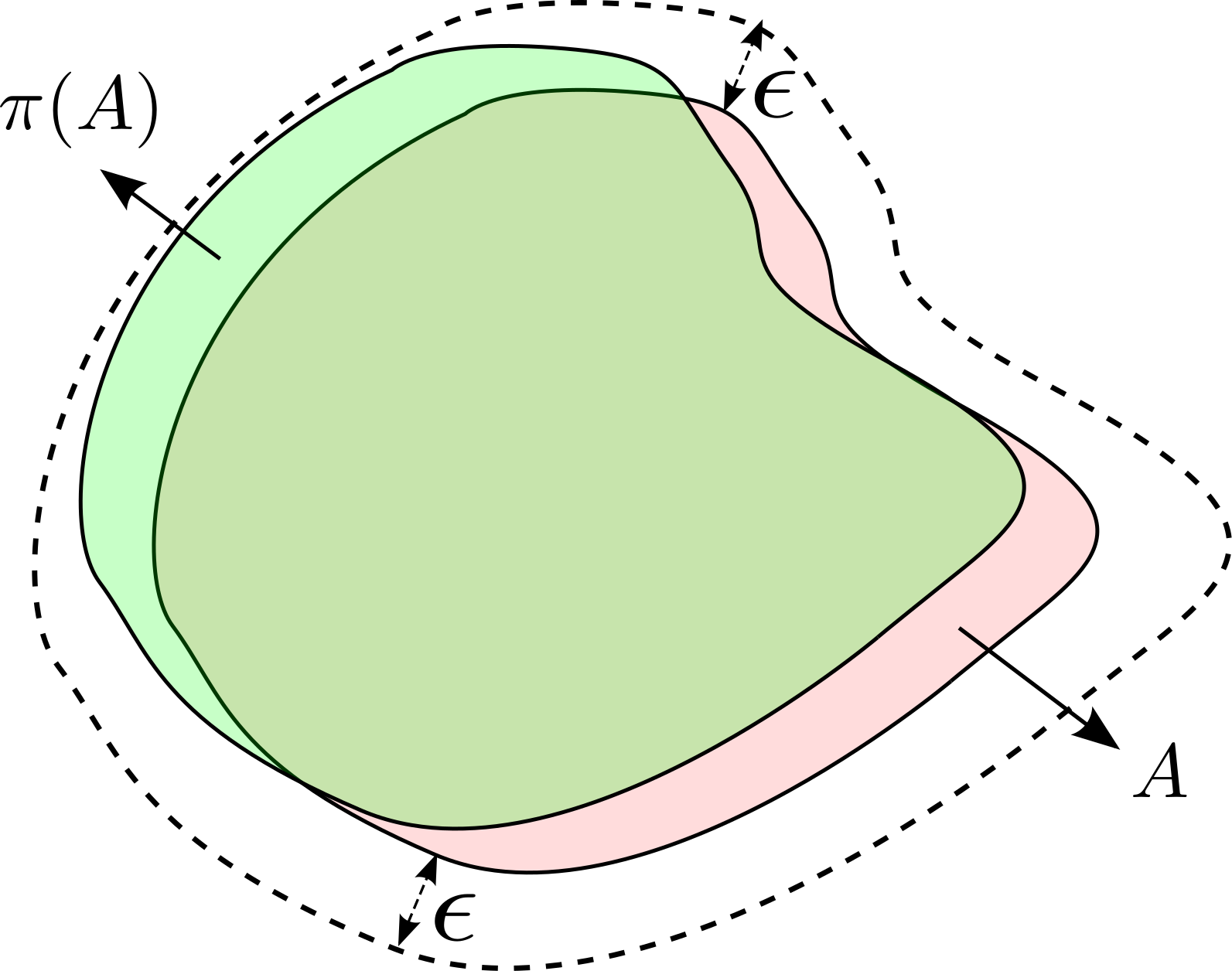}
\caption{\small Consider a mapping $\pi: x\mapsto x-\SmallVar \phi(x)$. If the conditions of Lemma~\ref{lemma:vector-field} hold, then we have $\pi(A)\subset A_\SmallVar$ and consequentely $\mu_f(\pi(A)) \leq \mu_f(A_\SmallVar)$. We use inequality~\eqref{eqn:cheeger-flow} to lower bound the restricted Cheeger constant.}\label{fig:flow}
\end{figure}

Lemma~\ref{lemma:vector-field} reduces the problem of lower bounding $\C_f(V)$ to the problem of finding a proper vector field $\phi$ and verifying its properties for all points $x\in V$. Informally, the quantity $\C_f(V)$ measures the chance of escaping the set $V$. The lemma shows that if we can construct an ``oracle'' vector field $\phi$, such that at every point $x\in V$ it gives the correct direction (i.e.~$-\phi(x)$) to escape $V$, but always stay in $K$, then we obtain a strong lower bound on $\C_f(V)$. This construction is merely for the theoretical analysis and doesn't affect the execution of the algorithm. 

\begin{proofsketch} The proof idea is illustrated in Figure~\ref{fig:flow}: by constructing a mapping $\pi: x\mapsto x-\SmallVar\phi(x)$ that satisfies the conditions of the lemma, we obtain $\pi(A)\subset A_\SmallVar$ for all $A\subset V$, and consequently $\mu_f(\pi(A)) \leq \mu_f(A_\SmallVar)$. Then we are able to lower bound the restricted Cheeger constant by:
\begin{align}\label{eqn:cheeger-flow}
	\C_f(V) \geq \liminf_{\SmallVar\searrow 0} \inf_{A\subset V}  \frac{\mu_f(\pi(A))-\mu_f(A)}{\SmallVar\, \mu_f(A)} = \liminf_{\SmallVar\searrow 0} \inf_{d A\subset V}  \frac{1}{\SmallVar} \Big(\frac{\mu_f(\pi(d A))}{\mu_f(d A)} - 1\Big),
\end{align}
where $d A$ is an infinitesimal of the set $V$. It can be shown that the right-hand side of inequality~\eqref{eqn:cheeger-flow} is equal to $\inf_{x\in V} \{\inprod{\phi(x)}{\nabla f(x)}-\dvg \phi(x)\}$, which establishes the lemma. See Appendix~\ref{sec:proof-lemma-vector-field} for a rigorous proof. 
\end{proofsketch}

\vspace{10pt}
Before demonstrating the applications of Lemma~\ref{lemma:vector-field}, we make several additional mild assumptions on the parameter space and on the function $f$.

\begin{assumption}[boundary condition and smoothness]\label{assumption:regularity-for-cheeger}
~~
\begin{itemize}
\item The parameter space $K$ is a $d$-dimensional ball of radius $r > 0$ centered at the origin. There exists $r_0 > 0$ such that for every point $x$ satisfying $\ltwos{x} \in [r-r_0, r]$, we have $\inprod{x}{\nabla f(x)} \geq \ltwos{x}$.
\item For some $\Gp, \Lp, \Hp > 0$, the function $f$ is third-order differentiable with $\ltwos{\nabla f(x)} \leq \Gp$, $\lncs{\nabla^2 f(x)} \leq \Lp$ and $\lncs{\nabla^2 f(x) - \nabla^2 f(y)} \leq \Hp \ltwos{x-y}$ for any $x,y\in K$.
\end{itemize}
\end{assumption}

The first assumption requires the parameter space to be an Euclidean ball and imposes a gradient condition on its boundary. This is made mainly for the convenience of theoretical analysis. We remark that for any function $f$, the condition on the boundary can be satisfied by adding a smooth barrier function $\rho(\ltwos{x})$ to it, where the function $\rho(t) = 0$ for any $t < r-2r_0$, but sharply increases on the interval $[r-r_0,r]$ to produce large enough gradients. The second assumption requires the function $f$ to be third-order differentiable. We shall relax the second assumption in Section~\ref{sec:erm}.

The following proposition describes a lower bound on $\C_{(\xi f)}(K\backslash U)$ when $f$ is a smooth function and the set $U$ consists of approximate stationary points. Although we shall prove a stronger result, the proof of this proposition is a good example for demonstrating the power of Lemma~\ref{lemma:vector-field}.

\begin{proposition}\label{prop:lower-bounded-gradient-norm}
Assume that Assumption~\ref{assumption:regularity-for-cheeger} holds. For any $\epsilon > 0$, define the set of $\epsilon$-approximate stationary points $U \defeq \{x\in K: \ltwos{\nabla f(x)} < \epsilon\}$. For any $\xi \geq 2 \Lp/\epsilon^2$, we have $\C_{(\xi f)}(K\backslash U) \geq \frac{\xi \epsilon^2}{2\Gp}$.
\end{proposition}

\begin{proof}
Recall that $\Gp$ is the Lipschitz constant of function $f$.
Let the vector field be defined by $\phi(x) \defeq \frac{1}{\Gp}\nabla f(x)$, then we have $\ltwos{\phi(x)} \leq 1$. By Assumption~\ref{assumption:regularity-for-cheeger}, it is easy to verify that the conditions of Lemma~\ref{lemma:vector-field} hold.
For any $x\in K\backslash U$, the fact that $\ltwos{\nabla f(x)} \geq \epsilon$ implies:
\[
	\inprod{\phi(x)}{\xi \nabla f(x)} = \frac{\xi}{\Gp} \ltwos{\nabla f(x)}^2 \geq \frac{\xi \epsilon^2}{\Gp}.
\]
Recall that $\Lp$ is the smoothness parameter. By Assumption~\ref{assumption:regularity-for-cheeger}, the divergence of $\phi(x)$ is upper bounded by $\dvg \phi(x) = \frac{1}{\Gp}\tr(\nabla^2 f(x)) \leq \frac{1}{\Gp}\lncs{\nabla^2 f(x)} \leq \frac{\Lp}{\Gp}$. Consequently, if we choose $\xi \geq 2 \Lp/\epsilon^2$ as assumed, then we have:
\[
	\inprod{\phi(x)}{\xi \nabla f(x)} - \dvg \phi(x)  \geq \frac{\xi \epsilon^2}{\Gp} - \frac{\Lp}{\Gp} \geq \frac{\xi \epsilon^2}{2\Gp}.
\]
Lemma~\ref{lemma:vector-field} then establishes the claimed lower bound.
\end{proof}

Next, we consider approximate local minima~\citep{nesterov2006cubic,agarwal2016finding}, which rules out local maxima and strict saddle points. For an arbitrary $\epsilon > 0$, the set of $\epsilon$-approximate local minima is defined by:
\begin{align}\label{eqn:define-approx-local-min}
U \defeq \{x\in K: \ltwos{\nabla f(x)} < \epsilon \mbox{~~and~~} \nabla^2 f(x) \succeq -\sqrt{\epsilon} I\}.
\end{align}
We note that an approximate local minimum is not necessarily close to any local minimum of $f$. However, if we assume in addition the the function satisfies the (robust) strict-saddle property~\citep{ge2015escaping, lee2016gradient}, then any point $x\in U$ is guaranteed to be close to a local minimum. Based on definition~\eqref{eqn:define-approx-local-min}, we prove a lower bound for the set of approximate local minima.

\begin{proposition}\label{prop:lower-bounded-strict-saddle}
Assume that Assumption~\ref{assumption:regularity-for-cheeger} holds. For any $\epsilon > 0$, let $U$ be the set of $\epsilon$-approximate local minima. For any $\xi$ satisfying 
\begin{align}\label{eqn:strict-saddle-xi-constraint}
\xi \geq \widetilde \order(1)\cdot  \frac{\max\{1, \Gp^{5/2}\Lp^5, \Hp^{5/2}\}}{\epsilon^2\Gp^{1/2}},
\end{align}
we have $\C_{(\xi f)}(K\backslash U) \geq \frac{\sqrt{\epsilon}}{8(2\Gp+1)\Gp}$. The notation $\widetilde \order(1)$ hides a poly-logarithmic function of $(\Lp,1/\epsilon)$.
\end{proposition}

\begin{proofsketch}
Proving Proposition~\ref{prop:lower-bounded-strict-saddle} is significantly more challenging than proving Proposition~\ref{prop:lower-bounded-gradient-norm}. From a high-level point of view, we still construct a vector field $\phi$, then lower bound the expression $\inprod{\phi(x)}{\xi \nabla f(x)} - \dvg \phi(x)$ for every point $x\in K\backslash U$ in order to apply Lemma~\ref{lemma:vector-field}. However, there exist saddle points in the set $K\backslash U$, such that the inner product $\inprod{\phi(x)}{\xi \nabla f(x)}$ can be very close to zero. For these points, we need to carefully design the vector field so that the  term $\dvg \phi(x)$ is strictly negative and bounded away from zero. To this end, 
we define $\phi(x)$ to be the sum of two components. The first component aligns with the gradient $\nabla f(x)$. The second component aligns with a projected vector $\Pi_x(\nabla f(x))$, which projects $\nabla f(x)$ to the linear subspace spanned by the eigenvectors of $\nabla^2 f(x)$ with negative eigenvalues. It can be shown that the second component produces a strictly negative divergence in the neighborhood of strict saddle points. 
See Appendix~\ref{sec:proof-prop-lower-bounded-strict-saddle} for the complete proof.
\end{proofsketch}

\subsection{Polynomial-time bound for finding an approximate local minimum}

Combining Proposition~\ref{prop:lower-bounded-strict-saddle} with Theorem~\ref{theorem:sld-general}, we conclude that SGLD finds an approximate local minimum of the function $f$ in polynomial time, assuming that $f$ is smooth enough to satisfy Assumption~\ref{assumption:regularity-for-cheeger}.

\begin{corollary}\label{coro:nu-non-convex}
Assume that Assumptions~\ref{assumption:conditions-for-szld},\ref{assumption:regularity-for-cheeger} hold. For an arbitrary $\epsilon > 0$, let $U$ be the set of $\epsilon$-approximate local minima. For any $\rho, \SmallProbParam > 0$, there exists a large enough $\xi$ and hyperparameters $(\eta,\kmax,\DistBound)$ such that with probability at least $1-\SmallProbParam$, SGLD returns a solution $\xhat$ satisfying 
\[f(\xhat) \leq \sup_{x:\,d(x,U)\leq \rho} f(x).
\]
The iteration number $\kmax$ is bounded by a polynomial function of all hyperparameters in the assumptions as well as $(\epsilon^{-1}, \rho^{-1},\log(1/\SmallProbParam))$.
\end{corollary}

Similarly, we can combine Proposition~\ref{prop:cheeger-for-convex-function} or Proposition~\ref{prop:lower-bounded-gradient-norm} with Theorem~\ref{theorem:sld-general}, to obtain complexity bounds for finding the global minimum of a convex function, or finding an approximate stationary point of a smooth function. 

Corollary~\ref{coro:nu-non-convex} doesn't specify any upper limit on the temperature parameter $\xi$. As a result, SGLD can be stuck at the worst approximate local minima. It is important to note that the algorithm's capability of escaping certain local minima relies on a more delicate choice of $\xi$. Given objective function $f$, we consider an arbitrary smooth function $F$ such that $\norms{f-F}_\infty \leq 1/\xi$. By Theorem~\ref{theorem:sld-general}, for any target subset $U$, the hitting time of SGLD can be controlled by lower bounding the restricted Cheeger constant $\C_{\xi f}(K\backslash U)$. By the stability property~\eqref{eqn:robustness-property}, it is equivalent to lower bounding $\C_{\xi F}(K\backslash U)$ because $f$ and $F$ are uniformly close. If $\xi > 0$ is chosen large enough (w.r.t.~smoothness parameters of $F$), then the lower bound established by Proposition~\ref{prop:lower-bounded-strict-saddle} guarantees a polynomial hitting time to the set $U_F$ of approximate local minima of~$F$. Thus, SGLD can efficiently escape all local minimum of $f$ that lie outside of $U_F$. Since the function $F$ is arbitrary, it can be thought as a favorable perturbation of $f$ such that the set $U_F$ eliminates as many local minima of $f$ as possible. The power of such perturbations are determined by their maximum scale, namely the quantity $1/\xi$. Therefore, it motivates choosing the smallest possible $\xi$ whenever it satisfies the lower bound in Proposition~\ref{prop:lower-bounded-strict-saddle}. 

The above analysis doesn't specify any concrete form of the function $F$. In Section~\ref{sec:erm}, we present a concrete analysis where the function $F$ is assumed to be the population risk of empirical risk minimization (ERM). We establish sufficient conditions under which SGLD efficiently finds an approximate local minima of the population risk.


\section{Applications to empirical risk minimization}
\label{sec:erm}

In this section, we apply SGLD to a specific family of functions, taking the form:
\[
f(x) \defeq \frac{1}{n} \sum_{i=1}^n \ell(x; a_i) \quad\mbox{for}\quad x\in K.
\]
These functions are generally referred as the \emph{empirical risk} in the statistical learning literature. Here, every instance $a_i\in \mathcal{A}$ is i.i.d.~sampled from a distribution $\mathcal{P}$, and the function $\ell: \R^d \times \mathcal{A}\to \R$ measures the loss on individual samples. We define \emph{population risk} to be the function $F(x) \defeq \E_{x\sim \mathcal{P}}[\ell(x,a)]$. 

We shall prove that under certain conditions, SGLD finds an approximate local minimum of the (presumably smooth) population risk in polynomial time, even if it is executed on a non-smooth empirical risk. More concretely, we run SGLD on a smoothed approximation of the empirical risk that satisfies Assumption~\ref{assumption:conditions-for-szld}. With large enough sample size, the empirical risk $f$ and its smoothed approximation will be close enough to the population risk $F$, so that combining the stability property with Theorem~\ref{theorem:sld-general} and Proposition~\ref{prop:lower-bounded-strict-saddle} establishes the hitting time bound. First, let's formalize the assumptions.

\begin{assumption}[parameter space, loss function and population risk]
\label{assumption:erm}
~
\begin{itemize}
\item The parameter space $K$ satisfies: there exists $\hmax > 0$, such that for any $x\in K$ and any $h\leq \hmax$, the random variable $y\sim N(x, 2h I)$ satisfies $P(y\in K) \geq \frac{1}{3}$.

\item There exist $\rhoK, \nu > 0$ such that in the set $\Kbar\defeq \{x: d(x,K)\leq \rhoK \}$, the population risk $F$ is $\Gp$-Lipschitz continuous, and $\sup_{x\in \Kbar}|f(x)-F(x)|\leq \nu$. 

\item For some $B > 0$, the loss $\ell(x;a)$ is uniformly bounded in $[0,B]$ for any $(x,a)\in \R^d\times \mathcal{A}$. 
\end{itemize}
\end{assumption}

The first assumption is identical to that of Assumption~\ref{assumption:conditions-for-szld}. The second assumption requires the population risk to be Lipschitz continuous, and it bounds the $\ell_\infty$-norm distance between $f$ and $F$. The third assumption requires the loss to be uniformly bounded. Note that Assumption~\ref{assumption:erm} allows the empirical risk to be non-smooth or even discontinuous.

Since the function $f$ can be non-differentiable, the stochastic gradient may not be well defined. We consider a smooth approximation of it following the idea of \citet{duchi2015optimal}:
\begin{align}\label{eqn:randomized-smoothing}
	\ftilde(x) \defeq \E_z[f(x+z)] \quad\mbox{where}\quad
	z\sim N(0, \sigma^2 I_{d\times d}),
\end{align}
where $\sigma > 0$ is a smoothing parameter. We can easily compute a stochastic gradient $g$ of $\ftilde$ as follows:
\begin{align}\label{eqn:gk-szld}
	\nabla \ftilde(x) = \E[g(x) \mid x] \quad  \mbox{where} \quad g(x) \defeq \frac{z}{\sigma^2}(\ell(x+z;a) - \ell(x;a)),
\end{align}
Here, $z$ is sampled from $N(0, \sigma^2 I_{d\times d})$ and $a$ is uniformly sampled from $\{a_1,\dots,a_n\}$. This stochastic gradient formulation is useful when the loss function $\ell$ is non-differentiable, or when its gradient norms are unbounded. The former happens for minimizing the zero-one loss, and the later can arise in training deep neural networks~\citep{pascanu2013difficulty,bengio2013advances}. Since the loss function is uniformly bounded, formula~\eqref{eqn:gk-szld} guarantees that the squared-norm $\ltwos{g(x)}^2$ is sub-exponential.

We run SGLD on the function $\ftilde$. Theorem~\ref{theorem:sld-general} implies that the time complexity inversely depends on the restricted Cheeger constant $\C_{(\xi \ftilde)}(K\backslash U)$. We can lower bound this quantity using $\C_{(\xi F)}(K\backslash U)$ --- the restricted Cheeger constant of the population risk. Indeed, by choosing a small enough $\sigma$, it can be shown that $\sup_{x\in K} |\ftilde(x)-F(x)|\leq 2\nu$. The stability property~\eqref{eqn:robustness-property} then implies
\begin{align}\label{eqn:bound-cheeger-by-nu}
	\C_{(\xi \ftilde)}(K\backslash U) \geq e^{-4\xi \nu}\,\C_{(\xi F)}(K\backslash U).
\end{align}
For any $\xi\in (0,1/\nu]$, we have $e^{-4\xi \nu} \geq e^{-4}$, thus the term $\C_{(\xi \ftilde)}(K\backslash U)$ is lower bounded by $e^{-4}\,\C_{(\xi F)}(K\backslash U)$. As a consequence, we obtain the following special case of Theorem~\ref{theorem:sld-general}.

\begin{theorem}\label{theorem:pop-general}
Assume that Assumptions~\ref{assumption:erm} holds. For any subset $U\subset K$, any $\SmallProbParam > 0$ and any $\xi\in (0,1/\nu]$, there exist hyperparameters $(\eta,\sigma,\kmax,\DistBound)$ such that with probability at least $1-\SmallProbParam$, running SGLD on $\ftilde$ returns a solution $\xhat$ satisfying:
\begin{align}\label{eqn:pop-optimality-gap}
	F(\xhat) \leq \sup_{x\in U} F(x) + 5\nu.
\end{align}
The iteration number $\kmax$ is polynomial in $(B, \log(1/\SmallProbParam), d, \hmax^{-1}, \nu^{-1}, \rhoK^{-1}, \C^{-1}_{(\xi F)} (K\backslash U))$.
\end{theorem}

\noindent
See Appendix~\ref{sec:proof-theorem-pop-general} for the proof.

In order to lower bound the restricted Cheeger constant $\C_{(\xi F)} (K\backslash U)$, we resort to the general lower bounds in Section~\ref{sec:cheeger-lower-bounds}. Consider population risks that satisfy the conditions of Assumption~\ref{assumption:regularity-for-cheeger}. By combining Theorem~\ref{theorem:pop-general} with Proposition~\ref{prop:lower-bounded-strict-saddle}, we conclude that SGLD finds an approximate local minimum of the population risk in polynomial time.

\begin{corollary}\label{coro:nu-non-convex-erm}
Assume that Assumption~\ref{assumption:erm} holds. Also assume that Assumption~\ref{assumption:regularity-for-cheeger} holds for the population risk~$F$ with smoothness parameters $(\Gp, \Lp, \Hp)$. For any $\epsilon > 0$, let $U$ be the set of $\epsilon$-approximate local minima of $F$. If
\begin{align}\label{eqn:nu-non-convex}
	\sup_{x\in K} |f(x)-F(x)| \leq \widetilde \order(1) \cdot \frac{\epsilon^2 \Gp^{1/2}}{\max\{1, \Gp^{5/2}\Lp^5, \Hp^{5/2}\}},
\end{align}
then there exist hyperparameters $(\xi, \eta,\sigma,\kmax,\DistBound)$ such that with probability at least $1-\SmallProbParam$, running SGLD on $\ftilde$ returns a solution $\xhat$ satisfying $F(\xhat) \leq \sup_{x\in U} F(x) + 5 \nu$.
The time complexity will be bounded by a polynomial function of all hyperparameters in the assumptions as well as $(\epsilon^{-1}, \log(1/\SmallProbParam))$. The notation $\widetilde \order(1)$ hides a poly-logarithmic function of $(\Lp,1/\epsilon)$.
\end{corollary}

Assumption~\ref{assumption:regularity-for-cheeger} requires the population risk to be sufficiently smooth. Nonetheless, assuming smoothness of the population risk is relatively mild, because even if the loss function is discontinuous, the population risk can be smooth given that the data is drawn from a smooth density. The generalization bound~\eqref{eqn:nu-non-convex} is a necessary condition, because the constraint $\xi \leq 1/\nu$ for Theorem~\ref{theorem:pop-general} and the constraint~\eqref{eqn:strict-saddle-xi-constraint} for Proposition~\ref{prop:lower-bounded-strict-saddle} must simultaneously hold.
With a large sample size $n$, the empirical risk can usually be made sufficiently close to the population risk. There are multiple ways to bound the $\ell_\infty$-distance between the empirical risk and the population risk, either by bounding the VC-dimension~\citep{vapnik1999overview}, or by bounding the metric entropy~\citep{haussler1992decision} or the Rademacher complexity~\citep{bartlett2003rademacher} of the function class. We note that for many problems, the function gap uniformly converges to zero in a rate $\order(n^{-c})$ for some constant $c > 0$. For such problems, the condition~\eqref{eqn:nu-non-convex} can be satisfied with a polynomial sample complexity.

\section{Learning linear classifiers with zero-one loss}
\label{sec:zero-one}

As a concrete application, we study the problem of learning linear classifiers with zero-one loss. The learner observes i.i.d.~training instances $(a,b)$ where $(a,b)\in \R^d\times\{-1,1\}$ are feature-label pairs. The goal is to learn a linear classifier $a \mapsto \inprod{x}{a}$ in order to minimize the zero-one loss:
\[
	F(x) \defeq \E_{(a,b)\sim \mathcal{P}}[\ell(x; (a,b))] \quad\mbox{where}\quad \ell(x; (a,b)) \defeq \left\{\begin{array}{ll}
	0 & \mbox{if } b\times \inprod{x}{a} > 0,\\
	1 & \mbox{if } b\times \inprod{x}{a} < 0,\\
	1/2 & \mbox{if } \inprod{x}{a} = 0,
	\end{array}\right.
\]
For a finite dataset $\{(a_i,b_i)\}_{i=1}^n$, the empirical risk is
$f(x) \defeq \frac{1}{n} \sum_{i=1}^n \ell(x; (a_i,b_i))$.
Clearly, the function $f$ is non-convex and discontinous, and has zero gradients almost everywhere. Thus the optimization cannot be accomplished by gradient descent.

For a general data distribution, finding a global minimizer of the population risk is NP-hard~\citep{arora1993hardness}. We follow~\citet{awasthi2015efficient} to assume that the feature vectors are drawn uniformly from  the unit sphere, and the observed labels $b$ are corrupted by the Massart noise. More precisely, we assume that there is an unknown unit vector $\xstar$ such that for every feature $a\in \R^d$, the observed label $b$ satisfies:
\begin{align}\label{eqn:define-massart-noise}
	b = \left\{\begin{array}{ll}
		\sgn(\inprod{\xstar}{a}) & \mbox{with probability } \frac{1+\Signal(a)}{2};\\
		- \sgn(\inprod{\xstar}{a}) & \mbox{with probability } \frac{1-\Signal(a)}{2};
	\end{array}\right.
\end{align}
where $\frac{1 - \Signal(a)}{2}\in [0,0.5]$ is the Massart noise level. 
We assume that the noise level is strictly smaller than $0.5$ when the feature vector $a$ is separated apart from the decision boundary. Formally, there is a constant $0 < \BaseSignal \leq 1$ such that
\begin{align}\label{eqn:massart-constraint}
	\Signal(a) \geq \BaseSignal\, |\inprod{\xstar}{a}|.
\end{align}
The value of $\Signal(a)$ can be adversarially perturbed  as long as it satisfies the constraint~\eqref{eqn:massart-constraint}. \citet{awasthi2015efficient} studied the same Massart noise model, but they impose a stronger constraint $\Signal(a) \geq 1 - 3.6\times 10^{-6}$ for all $a\in \R^d$, so that almost all observed labels are accurate. In contrast, our model~\eqref{eqn:massart-constraint} captures arbitrary Massart noises (because $\BaseSignal$ can be arbitrarily small), and allows for completely random observations at the decision boundary. Our model is thus more general than that of~\citet{awasthi2015efficient}.

Given function $f$, we use SGLD to optimize its smoothed approximation~\eqref{eqn:randomized-smoothing} in a compact parameter space $K\defeq \{x\in \R^d: 1/2 \leq \ltwos{x} \leq 1\}$. The following theorem shows that the algorithm finds an approximate global optimum in polynomial time, with a polynomial sample complexity.

\begin{theorem}\label{theorem:zero-one}
Assume that $d \geq 2$. For any $\BaseSignal\in (0,1]$ and $\epsilon,\SmallProbParam > 0$, if the sample size $n$ satisfies:
\[
	n \geq \widetilde O(1)\cdot \frac{d^4}{\BaseSignal^2 \epsilon^4},
\]
then there exist hyperparameters $(\xi,\eta,\sigma,\kmax,\DistBound)$ such that SGLD on the smoothed function~\eqref{eqn:randomized-smoothing} returns a solution $\xhat$ satisfying $F(\xhat) \leq F(\xstar) + \epsilon$ with probability at least $1-\SmallProbParam$. The notation $\widetilde O(1)$ hides a poly-logarithmic function of $(d,1/\BaseSignal,1/\epsilon,1/\SmallProbParam)$. The time complexity of the algorithm is polynomial in $(d,1/\BaseSignal, 1/\epsilon, \log(1/\SmallProbParam))$.
\end{theorem}

\begin{proofsketch} The proof consists of two parts. For the first part, we prove that the population risk is Lipschitz continuous and the empirical risk uniformly converges to the population risk, so that Assumption~\ref{assumption:erm} hold. For the second part, we lower bound the restricted Cheeger constant by Lemma~\ref{lemma:vector-field}. The proof is spiritually similar to that of Proposition~\ref{prop:lower-bounded-gradient-norm} or Proposition~\ref{prop:lower-bounded-strict-saddle}. We define~$U$ to be the set of approximately optimal solutions, and construct a vector field $\phi$ such that:
\[
	\phi(x) \propto \inprod{x}{\xstar}\,x - \ltwos{x}^2\, \xstar.
\]
By lower bounding the expression $\inprod{\phi(x)}{\nabla f(x)} - \dvg \phi(x)$ for all $x\in K\backslash U$, Lemma~\ref{lemma:vector-field} establishes a lower bound on the restricted Cheeger constant. The theorem is established by combining the two parts together and by Theorem~\ref{theorem:pop-general}. We defer the full proof to Appendix~\ref{sec:proof-theorem-zero-one}.
\end{proofsketch}

\section{Conclusion}
\label{sec:discussion}

In this paper, we analyzed the hitting time of the SGLD algorithm on non-convex functions. Our approach is different from existing analyses on Langevin dynamics~\citep{bubeck2015finite,dalalyan2016theoretical,
bonis2016guarantees,teh2016consistency, raginsky2017nonconvex}, which connect LMC to a continuous-time Langevin diffusion process, then  study the mixing time of the latter process. In contrast, we are able to establish polynomial-time guarantees for achieving certain optimality sets, regardless of the exponential mixing time.

For future work, we hope to establish stronger results on non-convex optimization using the techniques developed in this paper. Our current analysis doesn't apply to training over-specified models. For these models, the empirical risk can be minimized far below the population risk~\citep{safran2015quality}, thus the assumption of Corollary~\ref{coro:nu-non-convex-erm} is violated. In practice, over-specification often makes the optimization easier, thus it could be interesting to show that this heuristic actually improves the restricted Cheeger constant. Another open problem is avoiding poor population local minima. \citet{jin2016local} proved that there are many poor population local minima in training Gaussian mixture models. It would be interesting to investigate whether a careful initialization could prevent SGLD from hitting such bad solutions.

\bibliographystyle{abbrvnat}
\bibliography{bib}
  
\appendix

\section{Restricted Cheeger constant is strictly positive}
\label{sec:positive-cheeger}

In this appendix, we prove that under mild conditions, the restricted Cheeger constant for a convex parameter space is always strictly positive. Let $K$ be an arbitrary convex parameter space with diameter $D < +\infty$. \citet[][Theorem 2.6]{lovasz1993random} proved the following isoperimetric inequality: for any subset $A\subset K$ and any $\SmallVar > 0$, the following lower bound holds:
\begin{align}\label{eqn:lovasz1993}
	\frac{\vol(A_\SmallVar) - \vol(A)}{\SmallVar\, \min\{\vol(A),\vol(K\backslash A_\SmallVar)\}} \geq \frac{2}{D},
\end{align}
where $\vol(A)$ represents the Borel measure of set $A$. Let $f_0(x)\defeq 0$ be a constant zero function. By the definition of the function-induced probability measure, we have 
\begin{align}\label{eqn:volume-equation}
	\mu_{f_0}(A) = \frac{\vol(A)}{\vol(K)} \quad\mbox{for all}\quad
	A\subset K.
\end{align}
Combining the inequality~\eqref{eqn:lovasz1993} with equation~\eqref{eqn:volume-equation}, we obtain:
\[
	\frac{\mu_{f_0}(A_\SmallVar) - \mu_{f_0}(A)}{\SmallVar\, \mu_{f_0}(A)\,(1-\mu_{f_0}(A_\SmallVar))} \geq \frac{\mu_{f_0}(A_\SmallVar) - \mu_{f_0}(A)}{\SmallVar\, \min\{\mu_{f_0}(A),1-\mu_{f_0}(A_\SmallVar)\}} \geq \frac{2}{D}.
\]
If the set $A$ satisfies $A\subset V\subset K$, then $1-\mu_{f_0}(A_\SmallVar) \geq 1 - \mu_{f_0}(V_\SmallVar)$. Combining it with the above inequality, we obtain:
\[
	\frac{\mu_{f_0}(A_\SmallVar) - \mu_{f_0}(A)}{\SmallVar\, \mu_{f_0}(A)} \geq \frac{2(1 - \mu_{f_0}(V_\SmallVar))}{D} = \frac{2(\vol(K) - \vol(V_\SmallVar))}{D\,\vol(K)}.
\]
According to the definition of the restricted Cheeger constant, the above lower bound implies:
\begin{align}\label{eqn:lower-bound-for-zero-function}
	\C_{f_0}(V) \geq \frac{2(\vol(K) - \lim_{\SmallVar\to 0}\vol(V_\SmallVar))}{D\,\vol(K)}.
\end{align}
Consider an arbitrary bounded function $f$ satisfying $\sup_{x\in K} |f(x)|\leq B < +\infty$, combining the stability property~\eqref{eqn:robustness-property} and inequality~\eqref{eqn:lower-bound-for-zero-function}, we obtain:
\[
	\C_{f}(V) \geq e^{-2B}\times \frac{2(\vol(K) - \lim_{\SmallVar\to 0}\vol(V_\SmallVar))}{D\,\vol(K)}.
\]
We summarize the result as the following proposition.

\begin{proposition}
Assume that $K$ is a convex parameter space with finite diameter. Also assume that $V\subset K$ is a measurable set satisfying $\lim_{\SmallVar\to 0}\vol(V_\SmallVar) < \vol(K)$. For any bounded function $f:K\to \R$, the restricted Cheeger constant $\C_f(V)$ is strictly positive.
\end{proposition}

\section{Proof of Theorem~\ref{theorem:sld-general}}
\label{sec:proof-theorem-sld-general}

The proof consists of two parts. We first establish a general bound on the hitting time of Markov chains to a certain subset $U\subset K$, based on the notion of \emph{restricted conductance}. Then we prove that the hitting time of SGLD can be bounded by the hitting time of a carefully constructed time-reversible Markov chain.
This Markov chain runs a Metropolis-Hastings algorithm that converges to the stationary distribution $\mu_{\xi f}$. We prove that this Markov chain has a bounded restricted conductance, whose value is characterized by  the restricted Cheeger constant that we introduced in Section~\ref{sec:define-rcc}. Combining the two parts establishes the general theorem.

\subsection{Hitting time of Markov chains}

For an arbitrary Markov chain defined on the parameter space $K$, we represent the Markov chain by its \emph{transition kernel} $\pi(x,A)$, which gives the conditional probability that the next state satisfies $x_{k+1}\in A$ given the current state $x_k = x$. Similarly, we use $\pi(x,x')$ to represent the conditional probability $P(x_{k+1}=x'|x_k=x)$. If $\pi$ has a stationary distribution, then we denote it by $Q_\pi$. 

A Markov chain is call \emph{lazy} if $\pi(x,x)\geq 1/2$ for every $x\in K$, and is called \emph{time-reversible} if  it satisfies 
\[
\int_A \pi(x,B)Q_{\pi}(x) = \int_B \pi(x,A)Q_{\pi}(x) \quad \mbox{for any $A,B\subset K$}.
\]
If $(x_0,x_1,x_2,\dots)$ is a realization of the Markov chain $\pi$, then the \emph{hitting time} to some set $U\subset K$ is denoted by:
\[
	\tau_\pi(U) \defeq \min\{k: x_k\in U\}.
\]

For arbitrary subset $V\subset K$, we define the \emph{restricted conductance}, denoted by $\Phi_\pi(V)$, to be the following infinimum ratio:
\begin{align}\label{eqn:define-conductance}
\Phi_\pi(V)\defeq \inf_{A\subset V} \frac{\int_A \pi(x,K\backslash A) Q_\pi(x) dx}{Q_\pi(A)}.
\end{align}

Based on the notion of restricted conductance, we present a general upper bound on the hitting time. For arbitrary subset $U\subset K$, suppose that $\pitilde$ is an arbitrary Markov chain whose transition kernel is stationary inside $U$, namely it satisfies $\pitilde(x,x)=1$ for any $x\in U$. 
Let $(\xtilde_0,\xtilde_1,\xtilde_2,\dots)$ be a realization of the Markov chain $\pitilde$. We denote by $Q_k$ the probability distribution of $\xtilde_k$ at iteration $k$. In addition, we define a measure of closeness between any two Markov chains.

\begin{definition*}
For two Markov chains $\pi$ and $\pitilde$, we say that $\pitilde$ is $\ClosenessParam$-close to $\pi$ w.r.t.~a set $U$ if the following condition holds for any $x\in K\backslash U$ and any $A\subset K\backslash\{x\}$:
\begin{align}\label{eqn:define-closeness}
	\pi(x, A)\leq \pitilde(x, A) \leq (1+\ClosenessParam)  \pi(x, A).
\end{align}
\end{definition*}

\noindent
Then we are able to prove the following lemma.

\begin{lemma}\label{lemma:hitting-time-general-bound}
Let $\pi$ be a time-reversible lazy Markov chain with atom-free stationary distribution $\Qpi$. Assume that $\pitilde$ is $\ClosenessParam$-close to $\pi$ w.r.t.~$U$ where $\ClosenessParam \leq \frac{1}{4}\Phi_\pi(K\backslash U)$. If there is a constant $M$ such that the distribution $Q_0$ satisfies $Q_0(A) \leq M\,\Qpi(A)$ for any $A\subset K\backslash U$, then for any $\SmallProbParam > 0$, the hitting time of the Markov chain is bounded by:
\begin{align}\label{eqn:hitting-time-general-bound}
	\tau_\pitilde(U) \leq \frac{4 \log(M/\SmallProbParam)}{\Phi^2_\pi(K\backslash U)},
\end{align}
with probability at least $1-\SmallProbParam$.
\end{lemma}

See Appendix~\ref{sec:proof-theorem-hitting-time-general} for the proof of Lemma~\ref{lemma:hitting-time-general-bound}. The lemma shows that if the two chains $\pi$ and $\pitilde$ are sufficiently close, then
the hitting time of the Markov chain $\pitilde$ will be inversely proportional to the square of the restricted conductance of the Markov chain $\pi$, namely $\Phi_\pi(K\backslash U)$. Note that if the density function of distribution $Q_\pi$ is bounded, then by choosing $Q_0$ to be the uniform distribution over $K$, there exists a finite constant $M$ such that $Q_0(A) \leq M Q_\pi(A)$, satisfying the last condition of Lemma~\ref{lemma:hitting-time-general-bound}.

\subsection{Proof of the theorem}
\label{sec:subsection-for-the-main-theorem}

The SGLD algorithm initializes $x_0$ by the uniform distribution $\mu_{f_0}$ (with $f_0(x) \equiv 0$). Then at iteration $k \geq 1$, it performs the following update:
\begin{align}\label{eqn:noisy-sgd-update-0}
y_k = x_{k-1} - \eta \cdot g(x_{k-1}) + \sqrt{2\eta/\xi}\cdot w; \quad x_{k} = \left\{\begin{array}{ll}
	y_k & \mbox{if } y_k\in K\cap\ball(x_{k-1};4\sqrt{2\eta d / \xi}),\\
	x_{k-1} & \mbox{otherwise}.
\end{array}
\right.
\end{align}
We refer the particular setting $\xi = 1$ as the ``standard setting''. For the ``non-standard'' setting of $\xi\neq 1$, we rewrite the first equation as:
\begin{align*}
y_k = x_{k-1} - (\eta/\xi) \cdot (\xi g(x_{k-1})) + \sqrt{2(\eta/\xi)}\cdot w.
\end{align*}
This re-formulation reduces to the problem to the standard setting, with stepsize $\eta/\xi$ and objective function~$\xi f$. Thus it suffices to prove the theorem in the standard setting, then plug in the stepsize $\eta/\xi$ and the objective function $\xi f$ to obtain the general theorem. Therefore, we assume $\xi = 1$ and consider the sequence of points $(x_0,x_1,\dots)$ generated by:
\begin{align}\label{eqn:noisy-sgd-update}
y_k = x_{k-1} - \eta \cdot g(x_{k-1}) + \sqrt{2\eta}\cdot w; \quad x_{k} = \left\{\begin{array}{ll}
	y_k & \mbox{if } y_k\in K\cap\ball(x_{k-1};4\sqrt{2\eta d}),\\
	x_{k-1} & \mbox{otherwise}.
\end{array}
\right.
\end{align}

We introduce two additional notations: for arbitrary functions $f_1,f_2$, we denote the maximal gap $\sup_{x\in K} |f_1(x) - f_2(x)|$ by the shorthand $\norms{f_1-f_2}_\infty$. For arbitrary set $V\subset K$ and $\rho > 0$, we denote the super-set $\{x\in K: d(x,V)\leq \rho\}$ by the shorthand $V_\rho$. Then we prove the following theorem for the standard setting.

\begin{theorem}\label{theorem:noisy-sgd-hitting-time}
Assume that Assumption~\ref{assumption:conditions-for-szld} holds. Let $x_0$ be sampled from $\mu_{f_0}$ and let the Markov chain $(x_0,x_1,x_2,\cdots)$ be generated by update~\eqref{eqn:noisy-sgd-update}. Let $U\subset K$ be an arbitrary subset and let $\rhoU > 0$ be an arbitrary positive number. Let $\C \defeq \C_{f}(K\backslash U)$ be a shorthand notation. Then for any $\SmallProbParam > 0$ and any stepsize $\eta$ satisfying
\begin{align}\label{eqn:noisy-sgd-h-constraint}
	\eta \leq c\,\min\Big\{d\rhoU^2, \hmax, \frac{\bmax^2}{d}, \frac{1}{d (\Gg^2+\Lf)}, \frac{\C^2}{d^3 (\Gg^2+\Lf)^2} \Big\},
\end{align}
the hitting time to set $U_\rhoU$ is bounded by
\begin{align}\label{eqn:noisy-sgd-hitting-time-general-bound}
	\min\{k: x_k\in U_\rhoU\} \leq \frac{c'\,\big(\norms{f-f_0}_\infty + \log(1/\SmallProbParam)\big)}{\min\{1,\eta\, \C^2/d\} },
\end{align}
with probability at least $1-\SmallProbParam$. Here, $c,c' > 0$ are universal constants.
\end{theorem}

\noindent
Theorem~\ref{theorem:noisy-sgd-hitting-time} shows that if we choose $\eta\in (0,\eta_0]$, where $\eta_0$ is the right-hand side of inequality~\eqref{eqn:noisy-sgd-h-constraint}, then with probability at least $1-\SmallProbParam$, the hitting time to the set $U_\rhoU$ is bounded by
\begin{align*}
	\min\{k: x_k\in U_\rho\} \leq  \frac{c' \, d\,\big(B + \log(1/\SmallProbParam)\big)}{\eta\, \C^2 } = \frac{c' \,\big(B + \log(1/\SmallProbParam)\big)}{\min\{1,(\eta/\eta_0)\,\eta_0 \C^2/d\} }.
\end{align*}
Combining it with the definition of $\eta_0$, and with simple algebra, we conclude that $\min\{k: x_k\in U_\rho\} \leq \frac{M}{\min\{1,\C\}^4}$ where $M$ is polynomial in $(B, \Lf, \Gg, \log(1/\SmallProbParam), d, \eta_0/\eta, \hmax^{-1}, \bmax^{-1}, \rho^{-1})$. This establishes the iteration complexity bound. Whenever $x_k$ hits $U_\rhoU$, we have
\[
	f(\xhat) \leq f(x_k) \leq \sup_{x:\,d(x,U)\leq \rho} f(x),
\]
which establishes the risk bound. Thus, Theorem~\ref{theorem:noisy-sgd-hitting-time} establishes Theorem~\ref{theorem:sld-general} for the special case of $\xi =1$.

In the non-standard setting ($\xi\neq 1$), we follow the reduction described above to substitute $(\eta,f)$ in Theorem~\ref{theorem:noisy-sgd-hitting-time} with the pair $(\eta/\xi, \xi f)$. As a consequence, the quantity $\C$ is substituted with $\C_{(\xi f)}(K\backslash U)$, and $(B, \Lf, \Gg, \eta_0, \bmax)$ are substituted with $(\xi B,\xi\Lf, \xi\Gg, \eta_0/\xi, \bmax/\xi)$. Both the iteration complexity bound and the risk bound hold as in the standard setting, except that after the substitution, the numerator $M$ in the iteration complexity bound has an additional polynomial dependence on $\xi$. Thus we have proved the general conclusion of Theorem~\ref{theorem:sld-general}. 

\paragraph{Proof of Theorem~\ref{theorem:noisy-sgd-hitting-time}}

For the function $f: K\to \R^d$ satisfying Assumption~\ref{assumption:conditions-for-szld}, we define a time-reversible Markov chain represented by the following transition kernel $\pi_f$. Given any current state $x_k = x\in K$, the Markov chain draws a ``candidate state'' $y\in \R^d$ from the following the density function:
\begin{align}\label{eqn:define-q-xy}
q_x(y) \defeq \frac{1}{2}\delta_x(y) + \frac{1}{2}\cdot \frac{1}{(4\pi \eta)^{d/2}}\E\Big[e^{-\frac{\ltwos{y - x + \eta\cdot g(x)}^2}{4 \eta}} \mid x\Big]
\end{align}
where $\delta_x$ is the Dirac delta function at point $x$. The expectation is taken over the stochastic gradient $g$ defined in equation~\eqref{eqn:noisy-sgd-update}, conditioning on the current state $x$. Then for any candidate state $y\in K\cap\ball(x;4\sqrt{2\eta d})$, we accept the candidate state (i.e., $x_{k+1} = y$) with probability:
\begin{align}\label{eqn:define-alpha-xy}
	\alpha_x(y) \defeq \min\Big\{ 1, \frac{q_y(x)}{q_x(y)} e^{f(x)-f(y)} \Big\},
\end{align}
or reject the candidate state (i.e., $x_{k+1} = x$) with probability $1-\alpha_x(y)$. All candidate states $y\notin K\cap\ball(x;4\sqrt{2\eta d})$ are rejected (i.e., $x_{k+1} = x$). It is easy to verify that $\pi_f$ executes a Metropolis-Hastings algorithm. Therefore, it induces a time-reversible Markov chain, and its stationary distribution is equal to $\mu_f(x) \propto e^{-f(x)}$.

Given the subset $U_\rhoU \subset K$, we define an auxiliary Markov chain and its transition kernel $\pitilde_f$ as follow. Given any current state $\xtilde_k = x \in K$, the Markov chain proposes a candidate state $y\in \R^d$ through the density function $q_x(y)$ defined by equation~\eqref{eqn:define-q-xy}, then accepts the candidate state if and only if $x\notin U_\rhoU$ and $y\in K\cap\ball(x;4\sqrt{2\eta d})$. Upon acceptance, the next state of $\pitilde_f$ is defined to be $\xtilde_{k+1} =y$, otherwise $\xtilde_{x+1}=x$.
The Markov chains $\pitilde_f$ differs from $\pi_f$ only in their different probabilities for acccepting the candidate state $y$. If $x\in U_\rhoU$, then $\pi_f$ may accept $y$ with  probability $\alpha_x(y)$, but $\pitilde_f$ always rejects $y$. If $x\notin U_\rhoU$ and $y\in K\cap\ball(x;4\sqrt{2\eta d})$, then $\pi_f$ accepts $y$ with probability $\alpha_x(y)$, while $\pitilde_f$ accepts with probability~$1$.

Despite the difference in their definitions, we are able to show that the two Markov chains are $\ClosenessParam$-close, where $\ClosenessParam$ depends on the stepsize $\eta$ and the properties of the objective function. 

\begin{lemma}\label{lemma:closeness}
Assume that $0 < \eta \leq \frac{\bmax^2}{32 d}$ and Assumption~\ref{assumption:conditions-for-szld} hold. Then the Markov chain $\pitilde_f$ is $\ClosenessParam$-close to $\pi_f$ w.r.t.~$U_\rho$ with $\ClosenessParam = e^{33 \eta d (\Gg^2 + \Lf)} - 1$.
\end{lemma}

\noindent See Appendix~\ref{sec:proof-lemma-closeness} for the proof.

Lemma~\ref{lemma:closeness} shows that if we choose $\eta$ small enough, then $\ClosenessParam$ will be sufficiently small. Recall from Lemma~\ref{lemma:hitting-time-general-bound} that we need $\ClosenessParam \leq \frac{1}{4}\Phi_{\pi_f}(K\backslash U_\rhoU)$ to bound the Markov chain $\pi_f$'s hitting time to the set $U_\rhoU$. It means that $\eta$ has to be chosen based on the restricted conductance of the Markov chain $\pi_f$. Although calculating the restricted conductance of a Markov chain might be difficult, the following lemma shows that the restricted conductance can be lower bounded by the restricted Cheeger constant. 

\begin{lemma}\label{lemma:cheeger-to-conductance}
Assume that $\eta \leq \min\{\hmax, 16d\rho^2, \frac{\bmax^2}{32 d},\frac{1}{100 d (\Gg^2 + \Lf)}\}$ and Assumption~\ref{assumption:conditions-for-szld} hold. Then for any $V\subset K$, we have:
\[
	\Phi_{\pi_f}(V) \geq \frac{1}{192}( 1 - e^{- \frac{1}{4}\sqrt{\eta/d}\, \C_{f}(V_\rho)}).
\]
\end{lemma}

\noindent See Appendix~\ref{sec:proof-lemma-cheeger-to-conductance} for the proof.

By Lemma~\ref{lemma:closeness} and Lemma~\ref{lemma:cheeger-to-conductance}, we are able to choose a sufficiently small $\eta$ such that the Markov chains $\pi_f$ and $\pitilde_f$ are close enough to satisfy the conditions of Lemma~\ref{lemma:hitting-time-general-bound}. Formally, the following condition on $\eta$ is sufficient.

\begin{lemma}\label{lemma:constraint-on-h}
There exists a universal constant $c>0$ such that for any stepsize $\eta$ satisfying:
\begin{align}\label{eqn:final-constraint-on-h}
	\eta \leq c\,\min\Big\{\hmax, d\rhoU^2, \frac{\bmax^2}{d}, \frac{1}{d (\Gg^2+\Lf)}, \frac{\C^2}{d^3 (\Gg^2+\Lf)^2} \Big\},
\end{align}
the Markov chains $\pi_f$ and $\pitilde_f$ are $\ClosenessParam$-close with $\ClosenessParam \leq \frac{1}{4}\Phi_{\pi_f}(K\backslash U_\rhoU)$. In addition, the restricted conductance satisfies the lower bound $\Phi_{\pi_f}(K\backslash U_\rhoU) \geq \min\{\frac{1}{2},\frac{\sqrt{\eta/d}\, \C}{1536}\}$.
\end{lemma}
\noindent See Appendix~\ref{sec:proof-lemma-constraint-on-h} for the proof.

Under condition~\eqref{eqn:final-constraint-on-h}, the Markov chains $\pi_f$ and $\pitilde_f$ are $\ClosenessParam$-close with $\ClosenessParam \leq \frac{1}{4}\Phi_{\pi_f}(K\backslash U_\rhoU)$. Recall that the Markov chain $\pi_f$ is time-reversible and lazy. Since $f$ is bounded, the stationary distribution $Q_{\pi_f} = \mu_f$ is atom-free, and sampling $x_0$ from $Q_0 \defeq \mu_{f_0}$ implies:
\begin{align}
	Q_{0}(A) &= \frac{\int_A e^{-f_0(x)}dx}{\int_K e^{-f_0(x)}dx} \leq \frac{e^{\sup_{x\in K} f(x) - f_0(x)}\int_A e^{-f(x)}dx}{e^{\inf_{x\in K} f(x) - f_0(x)} \int_K e^{-f(x)}dx} \leq e^{2\norms{f-f_0}_\infty} Q_{\pi_f}(A).
\end{align}
Thus the last condition of Lemma~\ref{lemma:hitting-time-general-bound} is satisfied. Combining Lemma~\ref{lemma:hitting-time-general-bound} with the lower bound $\Phi_{\pi_f}(K\backslash U_\rhoU) \geq\min\{\frac{1}{2},\frac{\sqrt{\eta/d}\, \C}{1536}\}$ in Lemma~\ref{lemma:constraint-on-h},  it implies that with probability at least $1-\SmallProbParam > 0$, we have
\begin{align}\label{eqn:sgd-specific-hitting-time}
	\tau_{\pitilde_f}(U) \leq \frac{c'\, (\norms{f-f_0}_\infty + \log(1/\SmallProbParam))}{\min\{1,\eta\, \C^2/d\} },
\end{align}
where $c'>0$ is a universal constant.

Finally, we upper bound the hitting time of SGLD (i.e.,~the Markov chain induced by formula~\eqref{eqn:noisy-sgd-update}) using the hitting time upper bound~\eqref{eqn:sgd-specific-hitting-time}. We denote by $\pialg$ the transition kernel of SGLD, and claim that the Markov chain induced by it can be generated as a sub-sequence of the Markov chain induced by $\pitilde_f$. To see why the claim holds, we consider a Markov chain $(\xtilde_0,\xtilde_1,\xtilde_2,\dots)$ generated by $\pitilde_f$, and construct a sub-sequence $(x_0',x_1',x_2',\dots)$ of this Markov chain as follows:
\begin{enumerate}
\item Assign $x'_0 = \xtilde_0$ and initialize an index variable $\ell \leftarrow 0$. 
\item Examine the states $\xtilde_k$ in the order $k=1,2,\dots,\tau$, where $\tau = \min\{k: \xtilde_k \in U\}$:
\begin{itemize}
\item For any state $\xtilde_k$, in order to sample its next state $\xtilde_{k+1}$, the candidate state $y$ is either drawn from a delta distribution $\delta_{\xtilde_k}$, or drawn from a normal distribution with stochastic mean vector $x-\eta g(x)$. The probability of these two cases are equal, according to equation~\eqref{eqn:define-q-xy}.

\item If $y$ is drawn from the normal distribution, then generate a state $x'_{\ell+1} = \xtilde_{k+1}$ and add it to the sub-sequence $(x_0',x_1',x_2',\dots)$. Update the index variable $\ell \leftarrow \ell+1$.
\end{itemize}
\end{enumerate}
By this construction, it is easy to verify that $(x'_0,x_1',x'_2,\dots)$ is a Markov chain and its transition kernel exactly matches formula~\eqref{eqn:noisy-sgd-update}. Since the sub-sequence $(x'_0,x_1',x'_2,\dots)$ hits $U$ in at most $\tau$ steps, we have 
\[
\tau_{\pialg}(U) \leq \tau = \tau_{\pitilde_f}(U). 
\]
Combining this upper bound with~\eqref{eqn:sgd-specific-hitting-time} completes the proof of Theorem~\ref{theorem:noisy-sgd-hitting-time}.

\subsection{Proof of technical lemmas}

\subsubsection{Proof of Lemma~\ref{lemma:hitting-time-general-bound}}
\label{sec:proof-theorem-hitting-time-general}

Let $q\defeq Q_{\pi}(K\backslash U)$ be a shorthand notation.
Let $\Gfunc_p$ be the class of functions $g: K\backslash U\to [0,1]$ such that $\int_{K\backslash U} g(x) Q_{\pi}(x) dx = p$. We define a sequence of functions $h_k: [0, q] \to \R$ ($k=1,2,\dots$) such that
\begin{align}
	h_k(p) \defeq \sup_{g\in \Gfunc_p} \int_{K\backslash U} g(x) Q_k(x) dx.
\end{align}
By its definition, the function $h_k$ is a concave function on $[0,q]$. In addition, \cite[][Lemma 1.2]{lovasz1993random} proved the following properties for the function $h_k$: if $Q_{\pi}$ is atom-free, then for any $p\in [0,q]$ there exists a function $g(x) \defeq \indicator(x\in A)$ that attains the supremum in the definition of $h_k$. We claim the following property of the function $h_k$.

\begin{claim}\label{claim:h-function-recursion}
If there is a constant $C$ such that the inequality $h_{0}(p) \leq C\sqrt{p}$ holds for any $p\in[0,q]$, then the inequality 
\begin{align}\label{eqn:hk-recursion-bound}
	h_k(p) \leq C \sqrt{p}(1-\frac{1}{4}\Phi^2_\pi(K\backslash U))^{k}
\end{align}
holds for  any $k\in \N$ and any $p\in[0,q]$.
\end{claim}

According to the claim, it suffices to upper bound $h_{0}(p)$ for $p\in[0,q]$. Indeed, since $Q_0(A) \leq M\,\Qpi(A)$ for any $A\subset K\backslash U$, we immediately have:
\[
	h_0(p) = \sup_{A\subset K\backslash U: ~\Qpi(A) = p} Q_0(A)
	\leq M p \leq M \sqrt{p}.
\]
Thus, we have
\[
	Q_k(K\backslash U) \leq h_k(q) \leq M (1-\frac{1}{4}\Phi^2_\pi(K\backslash U))^{k}.
\]
Choosing $k \defeq \frac{4 \log(M/\SmallProbParam)}{\Phi^2_\pi(K\backslash U)}$ implies $Q_k(K\backslash U) \leq \SmallProbParam$. As a consequence, the hitting time is bounded by $k$ with probability at least $1-\SmallProbParam$. 

\paragraph{Proof of Claim~\ref{claim:h-function-recursion}}
Recall the properties of the function $h_k$.
For any $p\in [0,q]$, we can find a set $A\subset K\backslash U$ such that $\Qpi(A) = p$ and $h_k(p) = Q_k(A)$. 
Define, for $x\in K$, two functions:
\begin{align*}
&g_1(x) = \left\{\begin{array}{ll}
	2\pitilde(x, A)-1 & \mbox{if } x\in A\\
	0 & \mbox{if } x\notin A\\
\end{array}
\right.,\\
&g_2(x) = \left\{\begin{array}{ll}
	1 & \mbox{if } x\in A\\
	2\pitilde(x, A) & \mbox{if } x\notin A\\
\end{array}.
\right.,
\end{align*}
By the laziness of the Markov chain $\pitilde$, we obtain $0\leq g_i\leq 1$, so that they are functions mapping from $K\backslash U$ to $[0,1]$. Using the relation $2\pitilde(x, A)-1 = 1 - 2\pitilde(x, K\backslash A)$, the definition of $g_1$ implies that:
\begin{align}
\int_{K\backslash U} g_1(x) \Qpi(x) dx &= \Qpi(A) - 2 \int_A \pitilde(x, K\backslash A) \Qpi(x) dx = p - 2 \int_A \pitilde(x, K\backslash A) \Qpi(x) dx\nonumber\\
& \leq p - 2 \int_A \pi(x, K\backslash A) \Qpi(x) dx.\label{eqn:g1-integral-bound}
\end{align}
where the last inequality follows since the $\delta$-closeness ensures $\pi(x, K\backslash A) \leq \pitilde(x, K\backslash A)$. Similarly, using the definition of $g_2$ and the relation $\pitilde(x, A) \leq (1+\ClosenessParam) \pi(x, A)$, we obtain:
\begin{align}
\int_{K\backslash U} g_2(x) \Qpi(x) dx &= \Qpi(A) + 2 \int_{K\backslash (U\cup A)} \pitilde(x, A) \Qpi(x) dx\nonumber\\
& \leq p + 2 \int_{K\backslash A} \pitilde(x, A) \Qpi(x) dx\nonumber\\
& \leq p + 2 \int_{K\backslash A} (1+\ClosenessParam )\pi(x, A) \Qpi(x) dx\label{eqn:g2-integral-bound}
\end{align}
Since $Q_\pi$ is the stationary distribution of the time-reversible Markov chain $\pi$, the right-hand side of~\eqref{eqn:g2-integral-bound} is equal to:
\begin{align}\label{eqn:g2-integral-bound-2}
p + 2 \int_{K\backslash A} (1+\ClosenessParam )\pi(x, A) \Qpi(x) dx = p + 2 (1+\ClosenessParam ) \int_{A} \pi(x,K\backslash A) \Qpi(x) dx
\end{align}

Let $p_1$ and $p_2$ be the left-hand side of inequality~\eqref{eqn:g1-integral-bound} and~\eqref{eqn:g2-integral-bound} respectively, and define a shorthand notation:
\[
r\defeq \frac{1}{p}\int_{A} \pi(x, K\backslash A) \Qpi(x) dx.
\]
Then by definition of restricted conductance and the laziness of $\pi$, we have $\Phi_\pi(K\backslash U) \leq r \leq 1/2$. Combining inequalities~\eqref{eqn:g1-integral-bound},~\eqref{eqn:g2-integral-bound} and~\eqref{eqn:g2-integral-bound-2} and
by simple algebra, we obtain:
\begin{align*}
	\sqrt{p_1} + \sqrt{p_2} &\leq \sqrt{p}\Big(\sqrt{1 - 2r} + \sqrt{1 + 2\ClosenessParam r + 2r}\Big).
\end{align*}
By the condition $\ClosenessParam \leq \frac{1}{4}\Phi_\pi(K\backslash U) \leq \frac{r}{4}$, the above inequality implies
\begin{align*}
\sqrt{p_1} + \sqrt{p_2} &\leq \sqrt{p} \Big(\sqrt{1 - 2r} + \sqrt{1 + 2r + r^2/2}\Big)
\end{align*}
It is straightforward to verify that for any $0\leq r\leq 1$, the right-hand side is upper bounded by $2(1-r^2/4)\sqrt{p}$. Thus we obtain:
\begin{align}\label{eqn:p1-p2-sum-bound}
	\sqrt{p_1} + \sqrt{p_2} \leq 2(1-r^2/4)\sqrt{p}.
\end{align}

On the other hand, the definition of $g_1$ and $g_2$ implies that $\pitilde(x, A) = \frac{g_1(x) + g_2(x)}{2}$ for any $x\in K\backslash U$. For all $x\in U$, the transition kernel $\pitilde$ is stationary, so that we have $\pitilde(x, A) = 0$. Combining these two facts implies
\begin{align}\label{eqn:h-function-recursion}
	& h_k(p) = Q_k(A) = \int_{K\backslash U}  \pitilde(x, A) Q_{k-1}(x)dx\nonumber\\
	&= \frac{1}{2} \Big(\int_{K\backslash U} g_1(x) Q_{k-1}(x)dx + \int_{K\backslash U} g_2(x) Q_{k-1}(x)dx \Big) \leq \frac{1}{2}(h_{k-1}(p_1) + h_{k-1}(p_2)).
\end{align}
The last inequality uses the definition of function $h_{k-1}$.

Finally, we prove inequality~\eqref{eqn:hk-recursion-bound} by induction. The inequality holds for $k=0$ by the assumption. We assume by induction that it holds for an aribtrary integer $k-1$, and prove that it holds for $k$.  Combining the inductive hypothesis with inequalities~\eqref{eqn:p1-p2-sum-bound} and~\eqref{eqn:h-function-recursion}, we have
\begin{align*}
	h_k(p) &\leq \frac{C}{2}(\sqrt{p_1} + \sqrt{p_2})(1-\frac{1}{4}\Phi^2_\pi(K\backslash U))^{k-1} \leq C\sqrt{p}(1-r^2/4)(1-\frac{1}{4}\Phi^2_\pi(K\backslash U))^{k-1}\\
	&\leq C\sqrt{p}(1-\frac{1}{4}\Phi^2_\pi(K\backslash U))^{k},
\end{align*}
Thus, inequality~\eqref{eqn:hk-recursion-bound} holds for $k$, which completes the proof.


\subsubsection{Proof of Lemma~\ref{lemma:closeness}}
\label{sec:proof-lemma-closeness}

By the definition of the $\ClosenessParam$-closeness, it suffices to consider an arbitrary $x\notin U_\rhoU$ and verify the inequality~\eqref{eqn:define-closeness}. We focus on  cases when the acceptance ratio of $\pi_f$ and $\pitilde_f$ are different, that is, when the candidate state $y$ satisfies $y\neq x$ and $y\in K\cap \ball(x;4\sqrt{2\eta d})$. We make the following claim on the acceptance ratio.

\begin{claim}\label{claim:acceptance-ratio-bound}
For any $0 < \eta \leq \frac{\bmax^2}{32 d}$, if we assume $x\notin U_\rhoU$, $y\notin x$, and $y\in K\cap \ball(x;4\sqrt{2\eta d})$, then the acceptance ratio is lower bounded by $\alpha_x(y) \geq e^{- 33 \eta d (\Gg^2+\Lf)}$.
\end{claim}

Consider an arbitrary point $x\in K\backslash U_\rhoU$ and an arbitrary subset $A\subset K\backslash\{x\}$. The definitions of $\pi_f$ and $\pitilde_f$ imply that $\pi_f(x, A) \leq \pitilde_f(x, A)$ always hold. In order to prove the opposite, we notice that:
\begin{align}\label{eqn:pitilde-in-A-decompose}
	\pitilde_f(x, A) = \int_{A\cap \ball(x;4\sqrt{2\eta d})} q_x(y)dy.
\end{align}
The definition of $\pi_f$ and Claim~\ref{claim:acceptance-ratio-bound} implies
\begin{align*}
	\int_{A\cap \ball(x;4\sqrt{2\eta d})} q_x(y)dy &\leq e^{33 \eta d (\Gg^2+\Lf)} \int_{A\cap \ball(x;4\sqrt{2\eta d})} q_x(y)\alpha_x(y)dy\\
	&= e^{33 \eta d (\Gg^2+\Lf)}\;\pi(x,A),
\end{align*}
which completes the proof.

\paragraph{Proof of Claim~\ref{claim:acceptance-ratio-bound}}

By plugging in the definition of $\alpha_x(y)$ and $\alpha_y(x)$ and the fact that $x\neq y$, we obtain
\begin{align}\label{eqb:qx-qy-ratio-decompose}
\frac{q_y(x)}{q_x(y)} e^{f(x)-f(y)} = \frac{\E[e^{-\frac{\ltwos{x - y + \eta\cdot g(y)}^2}{4 \eta}} \mid y]}{\E[e^{-\frac{\ltwos{y - x + \eta\cdot g(x)}^2}{4 \eta}} \mid x]}
\cdot e^{f(x)-f(y)}.
\end{align}
In order to prove the claim, we need to lower bound the numerator and upper bound the denominator of equation~\eqref{eqb:qx-qy-ratio-decompose}.
For the numerator, Jensen's inequality implies:
\begin{align}
\E\Big[e^{-\frac{\ltwos{x - y + \eta\cdot g(y)}^2}{4 \eta}} \mid y\Big] &\geq e^{-\E[\frac{\ltwos{x - y + \eta\cdot g(y)}^2}{4 \eta} \mid y]} = e^{-\frac{\ltwos{x-y}^2}{4 \eta} - \E[\frac{\inprod{x-y}{g(y)}}{2} + \frac{\eta\ltwos{g(y)}^2}{4} | y]}\nonumber\\
&\geq e^{-\frac{\ltwos{x-y}^2}{4 \eta} - \frac{\inprod{x-y}{\nabla f(y)}}{2}  - \frac{\eta d \Gg^2}{4}}\label{eqb:qx-qy-ratio-numerator}
\end{align}
where the last inequality uses the upper bound
\begin{align}\label{eqn:gradient-norm-square-upper-bound}
	\E[\ltwos{g(y)}^2|y] &= \frac{1}{\bmax^2}\sum_{i=1}^d \E[(\bmax g_i(y))^2|y]
	\leq \frac{1}{\bmax^2}\sum_{i=1}^d \log\Big(\E[e^{(\bmax g_i(y))^2}|y]\Big)\nonumber\\
	&\leq \frac{1}{\bmax^2}\sum_{i=1}^d \log\Big(e^{\Gg^2\bmax^2}\Big) = d \Gg^2.
\end{align}
For the above deduction, we have used the Jensen's inequality as well as Assumption~\ref{assumption:conditions-for-szld}.

For the denominator, we notice that the term inside the expectation satisfies:
\begin{align}\label{eqn:upper-bound-denominator-e}
	e^{-\frac{\ltwos{y - x + \eta\cdot g(x)}^2}{4h}} \leq e^{-\frac{\ltwos{x-y}^2}{4h} - \frac{\inprod{y-x}{g(x)}}{2}}
\end{align}
Let $X$ be a shorthand for the random variable $\frac{\inprod{y-x}{g(x)}}{2}$. Using the relation that $e^t \leq t + e^{t^2}$ holds for all $t\in \R$,  we have
\begin{align}\label{eqn:upper-bound-mgf}
	\E[e^X] = e^{\E[X]}\E[e^{X-\E[X]}] = e^{\E[X]}\E[e^{X-\E[X]} - (X-\E[X])] \leq e^{\E[X]}\E[e^{(X-\E[X])^2}].
\end{align}
For the second term on the righthand side, using the relation $(a-b)^2 \leq 2a^2+2b^2$ and Jensen's inequality, we obtain
\begin{align}\label{eqn:upper-bound-e-variance}
	\E[e^{(X-\E[X])^2}] \leq \E[e^{2X^2+2(\E[X])^2}] = \E[e^{2X^2}]\cdot e^{2(\E[X])^2} \leq (\E[e^{2X^2}])^2 \leq \E[e^{4X^2}],
\end{align}
Since $\ltwos{x-y} \leq 4\sqrt{2 \eta d} \leq \bmax$ is assumed, Assumption~\ref{assumption:conditions-for-szld} implies
\begin{align}\label{eqn:upper-bound-e-x-square}
\E[e^{4X^2}] = \E[e^{(\inprod{x-y}{g(x)})^2}] \leq e^{\Gg^2\ltwos{x-y}^2} \leq e^{32 \eta d \Gg^2}.
\end{align}
Combining inequalities~\eqref{eqn:upper-bound-denominator-e}-\eqref{eqn:upper-bound-e-x-square}, we obtain
\begin{align}
	\E\Big[ e^{-\frac{\ltwos{y - x + \eta\cdot g(x)}^2}{4h}} \mid x\Big] \leq e^{-\frac{\ltwos{x-y}^2}{4h} - \frac{\inprod{y-x}{\nabla f(x)}}{2} + 32\eta d \Gg^2}.\label{eqb:qx-qy-ratio-denominator}
\end{align}
Combining equation~\eqref{eqb:qx-qy-ratio-decompose} with inequalities~\eqref{eqb:qx-qy-ratio-numerator},~\eqref{eqb:qx-qy-ratio-denominator}, we obtain
\begin{align}\label{eqb:qx-qy-ratio-lower-bound}
\frac{q_y(x)}{q_x(y)} e^{f(x)-f(y)} \geq e^{f(x)-f(y) - \inprod{x-y}{\frac{\nabla f(x) + \nabla f(y)}{2}} - 33\eta d \Gg^2}.
\end{align}
The $\Lf$-smoothness of function $f$ implies that 
\[
f(x)-f(y) - \inprod{x-y}{\frac{\nabla f(x) + \nabla f(y)}{2}} \geq -\frac{\Lf\ltwos{x-y}^2}{2} \geq - 16 \eta d \Lf.
\]
Combining this inequality with the lower bound~\eqref{eqb:qx-qy-ratio-lower-bound} completes the proof.


\subsubsection{Proof of Lemma~\ref{lemma:cheeger-to-conductance}}
\label{sec:proof-lemma-cheeger-to-conductance}

Recall that $\mu_f$ is the stationary distribution of the Markov chain $\pi_f$. We consider an arbitrary subset $A\subset V$, and define $B \defeq K\backslash A$. Let $A_1$ and $B_1$ be defined as
\[
	A_1 \defeq \{x\in A:~\pi_f(x, B) < 1/96\} \quad \mbox{and} \quad B_1 \defeq \{x\in B:~\pi_f(x, A) < 1/96\},
\]
In other words, the points in $A_1$ and $B_1$ have low probability to move across the broader between $A$ and $B$. We claim that the distance between points in $A_1$ and $B_1$ must be bounded away from a positive number.

\begin{claim}\label{claim:seperation-lemma}
Assume that $\eta \leq \min\{\hmax, \frac{\bmax^2}{32 d}, \frac{1}{100 d (\Gg^2 + \Lf)}\}$.
If $x\in A_1$ and $y\in B_1$, then $\ltwos{x-y} > \frac{1}{4}\sqrt{\eta/d}$.
\end{claim}

For any point $x\in K\backslash(A_1\cup B_1)$, we either have $x\in A$ and $\pi_f(x, B) \geq 1/96$, or we have $x\in B$ and $\pi_f(x, A) \geq 1/96$. It implies:
\begin{align}
	\mu_f(K\backslash(A_1\cup B_1)) &= \int_{A\backslash A_1} \mu_f(x)dx + \int_{B\backslash B_1} \mu_f(x)dx \nonumber \\
	&\leq \int_{A\backslash A_1} 96\pi_f(x, B) \mu_f(x) dx + \int_{B\backslash B_1} 96\pi_f(x, A) \mu_f(x)dx\nonumber\\
	&\leq \int_A 96\pi_f(x, B) \mu_f(x) dx + \int_{B} 96\pi_f(x, A) \mu_f(x)dx \label{eqn:pre-middle-region}
\end{align}
Since $\mu_f$ is the stationary distribution of the time-reversible Markov chain $\pi_f$, inequality~\eqref{eqn:pre-middle-region} implies:
\begin{align}\label{eqn:integrate-middle-region}
\int_A \pi_f(x, B)\mu_f(x)dx & = \frac{1}{2}\int_A \pi_f(x, B)\mu_f(x)dx + \frac{1}{2}\int_B \pi_f(x, A) \mu_f(x)dx \nonumber\\
	&\geq \frac{1}{192} \mu_f(K\backslash(A_1\cup B_1))  = \frac{1}{192} \Big(\mu_f(K\backslash B_1) - \mu_f(A_1) \Big).
\end{align}
Notice that $A\subset K\backslash B_1$, so that $\mu_f(K\backslash B_1) \geq \mu_f(A)$. According to Claim~\ref{claim:seperation-lemma}, by defining an auxiliary quantity:
\[
\rhoh \defeq \frac{1}{4}\sqrt{\eta/d},
\]
we find that the set $(A_1)_\rhoh$ belongs to $K\backslash B_1$, so that
$\mu_f(K\backslash B_1) \geq \mu_f((A_1)_\rhoh)$. The following property is a direct consequence of the definition of restricted Cheeger constant.

\begin{claim}\label{claim:prob-and-cheeger}
For any $A\subset V$ and any $\nu > 0$, we have $\mu_f(A_\nu) \geq e^{\nu\cdot \C_{f}(V_\nu)} \mu_f(A)$. 
\end{claim}

\noindent
Letting $A\defeq A_1$ and $\nu\defeq \rhoh$ in Claim~\ref{claim:prob-and-cheeger}, we have $\mu_f((A_1)_\rhoh) \geq e^{\rhoh \cdot \C_{f}(V_\rhoh)} \mu_f(A_1)$. Combining these inequalities, we obtain
\begin{align*}
	\mu_f(K\backslash B_1) - \mu_f(A_1) &\geq \max\Big\{\mu_f(A) - \mu_f(A_1), \Big(e^{\rhoh \cdot \C_{f}(V_\rhoh)} -1 \Big)\mu_f(A_1)\Big\}\\
	&\geq \Big( 1-e^{-\rhoh \cdot \C_{f}(V_\rhoh)}\Big) \mu_f(A),
\end{align*}
where the last inequality uses the relation $\max\{a-b,(\alpha-1)b\} \geq \frac{\alpha-1}{\alpha}(a-b) + \frac{1}{\alpha}(\alpha-1)b = \frac{\alpha-1}{\alpha}a$ with $\alpha \defeq e^{\rhoh\cdot\C_f(V_\rhoh)}$.
Combining it with inequality~\eqref{eqn:integrate-middle-region}, we obtain
\[
	\Phi_{\pi_f}(V) \geq \frac{1}{192} ( 1-e^{-\rhoh \cdot \C_{f}(V_\rhoh)}).
\]
The lemma's assumption gives $\rhoh = \frac{1}{4}\sqrt{\eta/d} \leq \rho$. Plugging in this relation completes the proof.

\paragraph{Proof of  Claim~\ref{claim:seperation-lemma}}

Consider any two points $x\in A$ and $y\in B$. Let $s$ be a number such that $2s\sqrt{2\eta d} = \ltwos{x-y}$. If $s > 1$, then the claim already holds for the pair $(x,y)$. Otherwise, we assume that $s\leq 1$, and as a consequence assume $\ltwos{x-y}\leq 2\sqrt{2\eta d}$.

We consider the set of points 
\[
Z\defeq \Big\{z\in \R^d\backslash\{x,y\}:~\ltwos{z-\frac{x+y}{2}} \leq 3\sqrt{2\eta d}\Big\}. 
\]
Denote by $q(z)$ the density function of distribution $N(\frac{x+y}{2}; 2\eta I)$. The integral $\int_{Z} q(z)dz$ is equal to $P(X \leq 9 d)$, where $X$ is a random variable satisfying the chi-square distribution with $d$ degrees of freedom. The following tail bound for the chi-square distribution was proved by~\citet{laurent2000adaptive}.

\begin{lemma}\label{lemma:chi-square-concentration}
If $X$ is a random variable satisfying the Chi-square distribution with $d$ degrees of freedom, then for any $x > 0$,
\begin{align*}
	P(X \geq d(1 + 2\sqrt{x} + 2x)) \leq e^{-x d}
	\quad \mbox{and} \quad 
	P(X \leq d(1-2\sqrt{x})) \leq e^{-x d}.
\end{align*}
\end{lemma}

\noindent
By choosing $x=9/5$ in Lemma~\ref{lemma:chi-square-concentration}, the probability $P(X \leq 9 d)$ is lower bounded by $1 - e^{-(9/5)d} > 5/6$. Since $\eta \leq \hmax$, the first assumption of Assumption~\ref{assumption:conditions-for-szld} implies $\int_K q(z)dz \geq 1/3$. Combining these two bounds, we obtain
\begin{align}\label{eqn:seperation-k-cap-z-prob}
\int_{K\cap Z} q(z)dz \geq \int_{K} q(z)dz - \int_{Z^c} q(z)dz > 1/6.
\end{align}

For any point $z\in Z$, the distances $\ltwos{z-x}$ and $\ltwos{z-y}$ are bounded by $4\sqrt{2\eta d}$. It implies
\begin{align}\label{eqn:Z-is-subset}
Z\subset \ball(x;4\sqrt{2\eta d})\cap \ball(y;4\sqrt{2\eta d}).
\end{align}
Claim~\ref{claim:acceptance-ratio-bound} in the proof of Lemma~\ref{lemma:closeness} demonstrates that the acceptance ratio $\alpha_x(z)$ and $\alpha_y(z)$ for any $z\in K\cap Z$ are both lower bounded by $e^{- 33 \eta d (\Gg^2 + \Lf)}$ given the assumption $0 < \eta \leq \frac{\bmax^2}{32 d}$. This lower bound is at least equal to $1/2$ because of the assumption $\eta \leq  \frac{1}{100 d (\Gg^2 + \Lf)}$, so that we have
\begin{align}\label{eqn:alpha-xz-alpha-yz-lower-bound}
	\alpha_x(z) \geq \frac{1}{2} \quad\mbox{and}\quad \alpha_y(z) \geq \frac{1}{2} \quad \mbox{for all} \quad z\in K\cap Z.
\end{align}

Next, we lower bound the ratio $q_x(z)/q(z)$ and $q_y(z)/q(z)$. For $z\in Z$ but $z\neq x$, the function $q_x(z)$ is defined by
\[
	q_x(z) = \frac{1}{2}\cdot \frac{1}{(4\pi \eta)^{d/2}}\E\Big[e^{-\frac{\ltwos{z - x + \eta\cdot g(x)}^2}{4\eta}} \mid x\Big],
\]
so that we have
\begin{align}
	\frac{q_x(z)}{q(z)} &= \frac{1}{2}\E\Big[\exp\Big(-\frac{\ltwos{z - x + \eta\cdot g(x)}^2 - \ltwos{z - \frac{x+y}{2}}^2}{4\eta}\Big)|x\Big]\nonumber\\
	&=\frac{1}{2}\E\Big[\exp\Big(-\frac{\inprod{\frac{y-x}{2}+\eta\cdot g(x)}{2(z - \frac{x+y}{2}) - \frac{x-y}{2} + \eta\cdot g(x)}}{4\eta}\Big)|x\Big].\nonumber\\
	&= \Big(\frac{1}{2}e^{-\frac{1}{4 \eta}\inprod{\frac{y-x}{2}}{2(z-\frac{x+y}{2}) + \frac{y-x}{2}}}\Big) \E[e^{-\frac{1}{4}\inprod{y-x + 2(z-\frac{x+y}{2})}{g(x)} - \frac{\eta}{4}\ltwos{g(x)}^2}|x]\nonumber\\
	&\geq \Big(\frac{1}{2} e^{-\frac{s(6+s)d}{2}}\Big) e^{-\E[\frac{1}{4}\inprod{y-x + 2(z-\frac{x+y}{2})}{g(x)} + \frac{\eta}{4}\ltwos{g(x)}^2|x]},\label{eqn:density-ratio-qx-and-q}
\end{align}
where the last inequality uses Jensen's inequality; It also uses the fact $\ltwos{\frac{y-x}{2}} = s\sqrt{2\eta d}$ and $\ltwos{z - \frac{x+y}{2}} \leq 3\sqrt{2\eta d}$. 

For any unit vector $u\in \R^d$, Jensen's inequality and Assumption~\ref{assumption:conditions-for-szld} imply:
\[
	\E[(\inprod{u}{g(x)})^2|x] \leq \frac{1}{\bmax^2}\log\Big(\E[e^{(\inprod{\bmax u}{g(x)})^2}|x]\Big) \leq \frac{1}{\bmax^2} \log(e^{\bmax^2 \Gg^2}) = \Gg^2.
\]
As a consequence of this upper bound and using Jensen's inequality, we have:
\begin{align}\label{eqn:inner-and-square-bounds}
	\E[\inprod{y-x + 2(z-\frac{x+y}{2})}{g(x)}|x] \leq \ltwos{y-x + 2(z-\frac{x+y}{2})} \Gg \leq (2s+6)\sqrt{2\eta d}\,\Gg.
\end{align}
Combining inequalities~\eqref{eqn:gradient-norm-square-upper-bound},~\eqref{eqn:density-ratio-qx-and-q} and~\eqref{eqn:inner-and-square-bounds}, we obtain:
\begin{align}\label{eqn:frac-qxz-qz}
	\frac{q_x(z)}{q(z)} \geq \frac{1}{2}e^{-\frac{s(6+s)d}{2} - \frac{(3+s)\sqrt{2\eta d}\,\Gg}{2} - \frac{\eta d \Gg^2}{4}}.
\end{align}
The assumption $\eta \leq \frac{1}{100 d(\Gg^2 + \Lf)}$ implies $\Gg \leq \frac{1}{10\sqrt{\eta d}}$. Plugging in this inequality to \eqref{eqn:frac-qxz-qz}, a sufficient condition for $q_x(z)/q(z) > 1/4$ is
\begin{align}\label{eqn:seperation-lemma-s-condition}
	s \leq \frac{1}{10 d}.
\end{align}
Following identical steps, we can prove that inequality~\eqref{eqn:seperation-lemma-s-condition} is a sufficient condition for $q_y(z)/q(z) > 1/4$ as well.

Assume that condition~\eqref{eqn:seperation-lemma-s-condition} holds. Combining inequalities~\eqref{eqn:seperation-k-cap-z-prob},~\eqref{eqn:alpha-xz-alpha-yz-lower-bound} with the fact $q_x(z) > q(z)/4$ and $q_y(z) > q(z)/4$, we obtain:
\begin{align}\label{eqn:seperation-lemma-min-prob-lower-bound}
	\int_{K\cap Z} \min\{q_x(z)\alpha_x(z), q_y(z)\alpha_y(z)\} dz \geq \frac{1}{8}\int_{K\cap Z} q(z) dz \geq \frac{1}{48}. 
\end{align}
Notice that the set $Z$ satisfies $Z\subset \ball(x;4\sqrt{2\eta d})\cap \ball(y;4\sqrt{2\eta d})$, thus the following lower bound holds:
\begin{align*}
	\pi_f(x, B) + \pi_f(y, A) &= \int_{B\cap\ball(x;4\sqrt{2\eta d})} q_x(z)\alpha_z(z) dz + \int_{A\cap\ball(y;4\sqrt{2\eta d})} q_y(z)\alpha_y(z) dz\\
	&\geq \int_{K\cap Z} \min\{q_x(z)\alpha_x(z), q_y(z)\alpha_y(z)\}dz \geq \frac{1}{48}.
\end{align*}
It implies that either $\pi_f(x, B) \geq \frac{1}{96}$ or $\pi_f(y, A) \geq \frac{1}{96}$. In other words, if $x\in A_1$ and $y\in B_1$, then inequality~\eqref{eqn:seperation-lemma-s-condition} \emph{must not} hold. As a consequence, we obtain the lower bound:
\[
\ltwos{x-y} = 2s\sqrt{2\eta d} \geq \frac{\sqrt{2}}{5}\sqrt{\eta/d} >  \frac{1}{4}\sqrt{\eta/d}.
\]

\paragraph{Proof of Claim~\ref{claim:prob-and-cheeger}}

Let $n$ be an arbitrary integer and let $i\in\{1,\dots,n\}$. By the definition of the restricted Cheeger constant (see equation~\eqref{eqn:cheeger-constant}), we have 
\[
	\log(\mu_f(A_{i\nu/n})) - \log(\mu_f(A_{(i-1)\nu/n})) \geq (\nu/n)(\C_{f}(V_\nu) - \SmallVar_n)\quad\mbox{for $i=1,\dots,n$}
\]
where $\SmallVar_n$ is an indexed variable satisfying $\lim_{n\to\infty} \SmallVar_n = 0$. Suming over $i=1,\dots,n$, we obtain 
\[
	\log(\mu_f(A_\nu) - \log(\mu_f(A)) \geq \nu\cdot (\C_{f}(V_\nu)-\SmallVar_n).
\]
Taking the limit $n\to\infty$ on both sides of the inequality completes the proof.


\subsubsection{Proof of Lemma~\ref{lemma:constraint-on-h}}
\label{sec:proof-lemma-constraint-on-h}

First, we impose the following constraints on the choice of $\eta$:
\begin{align}\label{eqn:h-constraints-to-satisfy-precondition}
\eta \leq\min\Big\{\hmax, 16d\rhoU^2, \frac{\bmax^2}{32 d}, \frac{1}{100 d(\Gg^2+\Lf)}\Big\},
\end{align}
so that the preconditions of both Lemma~\ref{lemma:closeness} and Lemma~\ref{lemma:cheeger-to-conductance} are satisfied. 
By plugging $V\defeq K\backslash U_\rho$ to Lemma~\ref{lemma:cheeger-to-conductance}, the restricted conductance is lower bounded by:
\begin{align}\label{eqn:apply-restricted-conductance-bound-pre}
\Phi_{\pi_f}(K\backslash U_\rhoU) \geq \frac{1}{192}(1 - e^{- \frac{1}{4}\sqrt{\eta/d}\cdot \C_f((K\backslash U_\rhoU)_\rhoU)}) \geq \min\Big\{\frac{1}{2}, \frac{\sqrt{\eta/d}\,\C_f((K\backslash U_\rhoU)_\rhoU)}{1536}\Big\}.
\end{align}
The last inequality holds because $1-e^{-t} \geq \min\{\frac{1}{2},\frac{t}{2}\}$ holds for any $t>0$. It is easy to verify that $(K\backslash U_\rhoU)_\rhoU \subset K\backslash U$, so that we have the lower bound
$\C_f((K\backslash U_\rhoU)_\rhoU) \geq \C_f(K\backslash U) = \C$. Plugging this lower bound to inequality~\eqref{eqn:apply-restricted-conductance-bound-pre}, we obtain
\begin{align}\label{eqn:apply-restricted-conductance-bound}
\Phi_{\pi_f}(K\backslash U_\rhoU) \geq \min\Big\{\frac{1}{2}, \frac{\sqrt{\eta/d}\,\C}{1536}\Big\}.
\end{align}
Inequality~\eqref{eqn:apply-restricted-conductance-bound} establishes the restricted conductance lower bound for the lemma.

Combining inequality~\eqref{eqn:apply-restricted-conductance-bound} with Lemma~\ref{lemma:closeness}, it remains to choose a small enough $\eta$ such that $\pitilde_f$ is $\ClosenessParam$-close to $\pi_f$ with $\ClosenessParam \leq \frac{1}{4}\Phi_\pi(K\backslash U_\rhoU)$. More precisely, it suffices to make the following inequality hold:
\[
	e^{33\eta d(\Gp^2 + \Lp)} - 1 \leq \frac{1}{4}\,\min\Big\{\frac{1}{2}, \frac{\sqrt{\eta/d}\,\C}{1536}\Big\}.
\]
In order to satisfy this inequality, it suffices to choose $\eta \lesssim \min\{\frac{1}{d(\Gp^2+\Lp)}, \frac{\C^2}{d^3 (\Gp^2 + \Lp)^2}\}$. Combining this result with~\eqref{eqn:h-constraints-to-satisfy-precondition} completes the proof.

\section{Proof of Proposition~\ref{prop:cheeger-for-convex-function}}
\label{sec:proof-prop-cheeger-for-convex-function}

\citet[][Theorem 2.6]{lovasz1993random} proved the following isoperimetric inequality:
Let $K$ be an arbitrary convex set with diameter $2$. For any convex function $f$ and any subset $V\subset K$ satisfying $\mu_{f}(V)\leq 1/2$, the following lower bound holds:
\begin{align}\label{eqn:lovasz-bound}
	\frac{\mu_f(A_\SmallVar) - \mu_f(A)}{\SmallVar\, \mu_f(A)} \geq 1 \quad\mbox{for all}\quad A\subset V,~\SmallVar > 0.
\end{align}
The lower bound~\eqref{eqn:lovasz-bound} implies $\C_{f}(V) \geq 1$. In order to establish the proposition, it suffices to choose $V\defeq K\backslash U$ and $f\defeq \xi f$, then prove the pre-condition $\mu_{\xi f}(K\backslash U) \leq 1/2$.

Let $\xstar$ be one of the global minimum of function $f$ and let $\ball(\xstar;r)$ be the ball of radius $r$ centering at point $\xstar$. If we choose $r = \frac{\epsilon}{2\Gp}$, then for any point $x\in \ball(\xstar;r)\cap K$, we have
\[
	f(x) \leq f(\xstar) + \Gp\ltwos{x-\xstar} \leq f(\xstar) + \epsilon/2.
\]
Moreover, for any $y\in K\backslash U$ we have:
\[
	f(y) \geq f(\xstar) + \epsilon.
\]
It means for the probability measure $\mu_{\xi f}$, the density function inside $\ball(\xstar;r)\cap K$ is at least $e^{\xi\epsilon/2}$ times greater than the density inside $K\backslash U$. It implies
\begin{align}\label{eqn:mu-xif-radio}
	\frac{\mu_{\xi f}(U)}{\mu_{\xi f}(K\backslash U)} \geq e^{\xi\epsilon/2}\,\frac{\vol(\ball(\xstar;r)\cap K)}{\vol(K\backslash U)} \geq e^{\xi \epsilon/2}\frac{\vol(\ball(\xstar;r)\cap K)}{\vol(K)}.
\end{align}

Without loss of generality, we assume that $K$ is the unit ball centered at the origin.
Consider the Euclidean ball $\ball(x';r/2)$ where $x' = \max\{0, 1 - r/(2\ltwos{\xstar})\}\xstar$. It is easy to verify that $\ltwos{x'} \leq 1-r/2$ and $\ltwos{x'-\xstar} \leq r/2$, which implies $\ball(x';r/2)\subset \ball(\xstar;r)\cap K$. Combining this relation with inequality~\eqref{eqn:mu-xif-radio}, we have
\[
	\frac{\mu_{\xi f}(U)}{\mu_{\xi f}(K\backslash U)} \geq e^{\xi \epsilon/2}\frac{\vol(\ball(x';r/2))}{\vol(K)} = e^{\xi\epsilon/2 - d \log (2/r)}
	= e^{\xi\epsilon/2 - d \log (4\Gp/\epsilon)}.
\]
The right-hand side is greater than or equal to $1$, because we have assumed $\xi \geq \frac{2 d \log(4\Gp/\epsilon)}{\epsilon}$. As a consequence, we have $\mu_{\xi f}(K\backslash U) \leq 1/2$.

\section{Proof of Lemma~\ref{lemma:vector-field}}
\label{sec:proof-lemma-vector-field}

Consider a sufficiently small $\SmallVar$ and a continuous mapping $\pi(x) \defeq x - \SmallVar \phi(x)$. Since $\phi$ is continuously differentiable in the compact set $K$, there exists a constant $G$ such that $\ltwos{\phi(x)-\phi(y)} \leq G\ltwos{x-y}$ for any $x,y\in K$. Assuming $\SmallVar < 1/G$, it implies
\[
	\ltwos{\pi(x)-\pi(y)} \geq \ltwos{x-y} - \SmallVar\,\ltwos{\phi(x)-\phi(y)} > 0 \quad \mbox{for any } x \neq y.
\]
Thus, the mapping $\pi$ is a continuous one-to-one mapping. 
For any set $A\subset K$, we define $\pi(A) \defeq \{\pi(s): x\in A\}$. 

Since the parameter set $K$ is compact, we can partition $K$ into a finite number of small compact subsets, such that each subset has diameter at most $\delta \defeq \SmallVar^2$. Let $S$ be the collection of these subsets that intersect with $A$. The definition implies $A \subset \cup_{B\in S} B \subset A_{\delta}$. 
The fact that $\ltwos{\phi(x)} \leq 1$ implies 
\begin{align*}
 \mu_f(\pi(\cup_{B\in S} B)) \leq \mu_f(\pi(A_{\delta})) \leq \mu_f(A_{\delta+\SmallVar}).
\end{align*}
As a consequence, we have:
\begin{align}\label{eqn:ratio-between-Ae-and-A-lower}
	\frac{\mu(A_{\delta+\SmallVar})}{\mu(A)} \geq \frac{\mu_f(\pi(\cup_{B\in S} B))}{\mu_f(\cup_{B\in S} B)} = \frac{\sum_{B\in S} \mu_f(\pi(B))}{\sum_{B\in S} \mu_f(B)} \geq \min_{B\in S} \frac{\mu_f(\pi(B))}{\mu_f(B)}.
\end{align}
For arbitrary $B\in S$, we consider a point $x\in B\cap A$, and remark that every point in $B$ is $\delta$-close to the point $x$. Since $\phi$ is continuously differentiable, the Jacobian matrix of the transformation $\pi$ has the following expansion:
\begin{align}\label{eqn:mapping-linear-shift}
	J(y) = I - \SmallVar H(x) + r_1(x,y)\quad \mbox{where $J_{ij}(x) = \frac{\partial \pi_i(x)}{\partial x_j}$},
\end{align}
where $H$ is the Jacobian matrix of $\phi$ satisfying $H_{ij}(x) = \frac{\partial \phi_i(x)}{\partial x_j}$. The remainder term $r_1(x,y)$, as a consequence of the continuous differentiability of $\phi$ and the fact $\ltwos{y-x}\leq \delta = \epsilon^2$, satisfies $\ltwos{r_1(x,y)} \leq C_1 \SmallVar^2$ for some constant $C_1$.

On the other hand, using the relation $\nabla \mu_f(y) = - \mu_f(y) \nabla f(y)$ and the continuous differentiability of $\mu_f$, the density function at $\pi(y)$ can be approximated by
\[
	\mu_f(\pi(y)) = \mu_f(y) + \nabla \mu_f(y) (\pi(y) - y) + r_2(y) = \Big(1 + \SmallVar\, \inprod{\phi(y)}{\nabla f(y)}\Big)\mu_f(y) + r_2(y),
\]
where the remainder term $r_2(y)$ satisfies $|r_2(y)| \leq C_2\SmallVar^2$ for some constant $C_2$. Further using the continuity of $\phi$, $\nabla f$ and the fact $\ltwos{y-x}\leq \epsilon^2$, we obtain:
\begin{align}\label{eqn:mapping-density-shift}
	\mu_f(\pi(y)) = \Big(1 + \SmallVar\, \inprod{\phi(x)}{\nabla f(x)}\Big)\mu_f(y) + r_3(x,y),
\end{align}
where the remainder term $r_3(x,y)$ satisfies $|r_3(x,y)| \leq C_3\SmallVar^2$ for some constant $C_3$.

Combining equation~\eqref{eqn:mapping-linear-shift} and equation~\eqref{eqn:mapping-density-shift}, we can quantify the measure of the set $\pi(B)$ using that of the set $B$. In particular, we have
\begin{align}
	\mu_f(\pi(B)) &= \int_B \mu_f(\pi(y)) d \pi(y) = \int_B \mu_f(\pi(y)) \det(J(y)) dy \nonumber\\
	&=\int_B \Big\{\big(1 + \SmallVar\, \inprod{\phi(x)}{\nabla f(x)}\big)\mu_f(y) + r_3(x,y)\Big\} 
	\det(I - \SmallVar H(x) + r_1(x,y)) dy\nonumber\\
	&= \Big(1 + \SmallVar\, \inprod{\phi(x)}{\nabla f(x)} - \SmallVar\,\tr(H(x)) \Big) \int_B \mu_f(y) dy + \order(\epsilon^2).\nonumber\\
	&= \Big(1 + \SmallVar\, \inprod{\phi(x)}{\nabla f(x)} - \SmallVar\,\tr(H(x)) \Big) \mu_f(B) + \order(\epsilon^2).\label{eqn:connect-mu-piB-and-mu-B}
\end{align}
Plugging equation~\eqref{eqn:connect-mu-piB-and-mu-B} to the lower bound~\eqref{eqn:ratio-between-Ae-and-A-lower} and using the relation $\tr(H(x)) = \dvg \phi(x)$, implies
\begin{align*}
	\frac{\mu(A_{\delta+\SmallVar})}{\mu(A)} - 1 &\geq \min_{B\in S}\Big(\inprod{\phi(x)}{\nabla f(x)}-\dvg \phi(x) \Big) \SmallVar + \order(\epsilon^2)\\
	&\geq \inf_{x\in A}\Big(\inprod{\phi(x)}{\nabla f(x)}-\dvg \phi(x) \Big) \SmallVar + \order(\epsilon^2),
\end{align*}
Finally, plugging in the definition of the restricted Cheeger constant and taking the limit $\SmallVar\to 0$ completes the proof.

\section{Proof of Proposition~\ref{prop:lower-bounded-strict-saddle}}
\label{sec:proof-prop-lower-bounded-strict-saddle}

\paragraph{Notations} Let $\Phi$ denote the CDF of the standard normal distribution. The function $\Phi$ satisfies the following tail bounds:
\begin{align}\label{eqn:cdf-pre-tail-bounds}
	0\leq \Phi(t) \leq e^{-t^2/2} \quad \mbox{for any}\quad  t\leq 0.
\end{align}
We define an auxiliary variable $\sigma$ based on the value of $\epsilon$:
\begin{align}\label{eqn:define-delta-sigma-eta}
 \sigma \defeq \frac{1}{2\sqrt{\log(4\Lp/\sqrt{\epsilon})}}.
\end{align}
Since $e^{-1/(2\sigma^2)} = \frac{\sqrt{\epsilon}}{4\Lp}$, the tail bound~\eqref{eqn:cdf-pre-tail-bounds} implies $\Phi(t)\leq \frac{\sqrt{\epsilon}}{4\Lp}$  for all $t \leq -\frac{1}{\sigma}$.

\paragraph{Define a vector field}

Let $g(x)\defeq \ltwos{\nabla f(x)}$ be a shorthand notation. We define a vector field:
\begin{align}\label{eqn:define-phi-mapping-for-saddle}
	\phi(x) &\defeq \frac{1}{(2\Gp+1)\Gp}\Big(\underbrace{2 \sqrt{\Gp\,g(x)}\, I + \Phi\Big(\frac{- \saddleconst I - \nabla^2 f(x)}{\sigma \saddleconst}\Big)}_{\defeq A(x)}\Big)  \nabla f(x)
\end{align}
Note that the function $\Phi$ admits a polynomial expansion:
\[
	\Phi(x) = \frac{1}{2} + \frac{1}{\sqrt{\pi}} \sum_{j=0}^\infty \frac{(-1)^j x^{2j+1}}{j!(2j+1)}
\]
Therefore, for any symmetric matrix $A\in \R^{d\times d}$, the matrix $\Phi(A)$ is well-defined by:
\begin{align}\label{eqn:define-phi-matrix}
	\Phi(A) \defeq \frac{I}{2} + \frac{1}{\sqrt{\pi}} \sum_{j=0}^\infty \frac{(-1)^j A^{2j+1}}{j!(2j+1)},
\end{align}
We remark that the matrix definition~\eqref{eqn:define-phi-matrix} implies  $\Phi(A+dA) = \Phi(A) + \Phi'(A) dA$ where $\Phi'$ is the derivative of function~$\Phi$. 

\paragraph{Verify the condition of Lemma~\ref{lemma:vector-field}}
The matrix $A(x)$ satisfies $0 \preceq A(x) \preceq (2\Gp+1)I$, so that $\ltwos{\phi(x)}\leq 1$ holds. For points that are $r_0$-close to the boundary, we have $\inprod{x}{\nabla f(x)} \geq \ltwos{x}$. By these lower bounds and definition~\eqref{eqn:define-phi-mapping-for-saddle}, we obtain:
\begin{align}
	\ltwos{x - \SmallVar\,\phi(x)}^2 &\leq \ltwos{x}^2 + \SmallVar^2 - \frac{\SmallVar}{(2 \Gp +1)\Gp} \Big(2\sqrt{\Gp\,g(x)}\inprod{x}{\nabla f(x)} - \ltwos{x}\cdot \ltwos{\nabla f(x)}\Big)\nonumber\\
	&\leq \ltwos{x}^2 + \SmallVar^2 - \frac{\SmallVar \ltwos{x}}{(2\Gp+1)\Gp} \Big(2 \sqrt{\Gp\,g(x)} - g(x)\Big).\nonumber\\
	&\leq \ltwos{x}^2 + \SmallVar^2 - \frac{\SmallVar \ltwos{x} \sqrt{\Gp\,g(x)}}{(2\Gp+1)\Gp}
	\leq \ltwos{x}^2+\SmallVar^2 - \frac{\SmallVar \ltwos{x}}{(2\Gp+1)\sqrt{\Gp}},
\end{align} 
where the last inequality holds because $g(x) \geq \inprod{x}{\nabla f(x)}/\ltwos{x} \geq 1$.
For any $\SmallVar < \frac{r-r_0}{(2\Gp+1)\sqrt{\Gp}}$, the right-hand side is smaller than $\ltwos{x}^2$, so that  $x - \SmallVar\,\phi(x) \in K$. For points that are not $r_0$-close to the boundary, we have $x - \SmallVar\,\phi(x) \in K$ given $\SmallVar < r_0$. Combining results for the two cases, we conclude that $\phi$ satisfies the conditions of Lemma~\ref{lemma:vector-field}

\paragraph{Prove the Lower bound} By applying Lemma~\ref{lemma:vector-field}, we obtain the following lower bound:
\begin{align}\label{eqn:expr-lower-bound-of-cheeger-lower-bound}
	\C_{(\xi f)}(K\backslash U) \geq \frac{1}{(2\Gp+1)\Gp}\inf_{x\in K\backslash U} \Big\{ \xi\,(\nabla f(x))^\top A(x) \nabla f(x) - \dvg A(x)\nabla f(x) \Big\}.
\end{align}
Since $A(x) \succeq 2\sqrt{\Gp\,g(x)} I$, the term $(\nabla f(x))^\top A(x) \nabla f(x)$ is lower bounded by $2\sqrt{\Gp}(g(x))^{5/2}$. 
For the second term, we claim the following bound:
\begin{align}\label{eqn:hard-dvg-upper-bound}
	\dvg A(x)\nabla f(x) \leq 3\sqrt{\Gp\,g(x)}L + \frac{g(x) \Hp}{\sigma \saddleconst} + \frac{\saddleconst}{4} - \indicator[g(x)< \epsilon]\,\frac{\saddleconst}{2}
\end{align}
We defer the proof to Appendix~\ref{sec:proof-eqn-hard-dvg-upper-bound} and focus on its consequence. Combining inequalities~\eqref{eqn:expr-lower-bound-of-cheeger-lower-bound} and~\eqref{eqn:hard-dvg-upper-bound}, we obtain
\begin{align}
	\C_{(\xi f)}(K\backslash U) \geq &\frac{1}{(2\Gp+1)\Gp} \inf_{x\in K\backslash U} \nonumber\\
	&\Big\{2\sqrt{\Gp} \xi (g(x))^{5/2} + \indicator[g(x)< \epsilon]\,\frac{\saddleconst}{2} - 3\sqrt{\Gp\,g(x)}L - \frac{g(x) \Hp}{\sigma \saddleconst} - \frac{\saddleconst}{4}\Big\}.\label{eqn:cheeger-lower-bound-cs}
\end{align}
The right-hand side of inequality~\eqref{eqn:cheeger-lower-bound-cs} can be made strictly positive if we choose a large enough $\xi$. In particular, we choose:
\begin{align}\label{eqn:saddle-assign-xi}
	\xi \geq \frac{1}{\epsilon^2 \Gp^{1/2}}\cdot \max\Big\{\frac{6 \Gp^{1/2} \Lp }{h^2}, \frac{2\Hp}{\sigma h^{3/2}}, \frac{1}{2 h^{5/2}}\Big\} 
	\quad\mbox{where}\quad 
	h \defeq \min\Big\{1, \frac{1}{\Gp(48\Lp)^2}, \frac{\sigma }{16 \Hp}\Big\}.
\end{align}
To proceed, we do a case study based on the value of $g(x)$. For all $x$ satisfying $g(x) < h \epsilon$, we plug in the upper bound $g(x) < h \epsilon$ for $g(x)$, then plug in the definition of $h$. It implies:
\begin{align}
	&\indicator[g(x)< \epsilon]\,\frac{\saddleconst}{2} - 3\sqrt{\Gp\,g(x)}L - \frac{g(x) \Hp}{\sigma \saddleconst} - \frac{\saddleconst}{4} \nonumber\\
	& \geq \frac{\saddleconst}{4} - 3\sqrt{\Gp}\Lp\cdot \sqrt{\frac{\epsilon}{\Gp(48\Lp)^2}} - \frac{\Hp}{\sigma \saddleconst}\cdot\frac{\epsilon\sigma}{16\Hp} = \frac{\saddleconst}{8}.\label{eqn:saddle-small-gx}
\end{align}
For all $x$ satisfying $g(x) \geq h \epsilon$, we ignore the non-negative term $\indicator[g(x)< \epsilon]\,\frac{\saddleconst}{2}$ on the right-hand side of~\eqref{eqn:cheeger-lower-bound-cs}, then re-arrange the remaining terms. It gives:
\begin{align*}
	& 2 \sqrt{\Gp} \xi (g(x))^{5/2} - 3\sqrt{\Gp\,g(x)}L - \frac{g(x) \Hp}{\sigma \saddleconst} - \frac{\saddleconst}{4} \geq  \frac{\xi \sqrt{\Gp} (g(x))^{5/2}}{2} + \sqrt{\Gp\,g(x)}\Big( \frac{\xi (g(x))^2}{2} -3\Lp\Big)\\
	&  \qquad\qquad +  g(x)\Big( \frac{\xi \sqrt{\Gp}(g(x))^{3/2}}{2} - \frac{\Hp}{\sigma \saddleconst} \Big) 
	+ \Big( \frac{\xi \sqrt{\Gp}(g(x))^{5/2}}{2} - \frac{\saddleconst}{4}\Big).
\end{align*}
Using lower bound~\eqref{eqn:saddle-assign-xi} for $\xi$ and the lower bound $g(x)\geq h \epsilon$ for $g(x)$, it is easy to verify that the last three terms on the right-hand side are non-negative. Furthermore, plugging in the lower bound $\xi \geq \frac{1}{\epsilon^2\Gp^{1/2}}\cdot\frac{1}{2 h^{5/2}}$ from \eqref{eqn:saddle-assign-xi}, it implies:
\begin{align}\label{eqn:saddle-big-gx}
	2 \sqrt{\Gp} \xi (g(x))^{5/2} - 3\sqrt{g(x)}L - \frac{g(x) \Hp}{\sigma \saddleconst} - \frac{\saddleconst}{4} \geq \sqrt{\Gp} \xi \cdot \frac{(h\epsilon)^{5/2}}{2} \geq \frac{1}{2\epsilon^2 h^{5/2}}\cdot\frac{(h\epsilon)^{5/2}}{2} = \frac{\sqrt{\epsilon}}{4}
\end{align}
Combining inequalities~\eqref{eqn:cheeger-lower-bound-cs},~\eqref{eqn:saddle-small-gx},~\eqref{eqn:saddle-big-gx} proves that the restricted Cheeger constant is lower bounded by $\frac{\sqrt{\epsilon}}{8(2\Gp+1)\Gp}$. Since $1/\sigma = \widetilde\order(1)$, it is easy to verify that the constraint~\eqref{eqn:saddle-assign-xi} can be satisfied if we choose:
\[
	\xi \geq \widetilde \order(1)\cdot  \frac{\max\{1, \Gp^{5/2}\Lp^5, \Hp^{5/2}\}}{\epsilon^2 \Gp^{1/2}},
\]
which completes the proof.

\subsection{Proof of inequality~\eqref{eqn:hard-dvg-upper-bound}}
\label{sec:proof-eqn-hard-dvg-upper-bound}

Let $T(x)$ be the third order tensor such that $T_{ijk}(x) = \frac{\partial^3 f(x)}{\partial x_i \partial x_j \partial x_k}$. Consider an arbitrary unit vector $u\in \R^d$. By the definition of $A(x)\nabla f(x)$, we have:
\begin{align*}
& \lim_{t\to 0}\frac{A(x+ut)\nabla f(x+ut) - A(x)\nabla f(x)}{t} = A(x) \nabla^2 f(x)\,u \\
& \qquad  + \frac{\sqrt{\Gp} \nabla f(x) (\nabla g(x))^\top}{\sqrt{g(x)}}  u - \frac{1}{\sigma\saddleconst} \Phi'\Big(\frac{- \saddleconst I - \nabla^2 f(x)}{\sigma \saddleconst}\Big) T(x)[u] \nabla f(x),
\end{align*}
where the matrix $T(x)[u]\in \R^{d\times d}$ is defined by $(T(x)[u])_{ij} = \sum_{k=1}^d T_{ijk}(x) u_k$. 
 By simple algebra, we obtain $T(x)[u] \nabla f(x) = T(x)[\nabla f(x)] u$. Thus, the derivative of the vector field $A(x)\nabla f(x)$ can be represented by $D_{ij} = \frac{\partial (A(x)\nabla f(x))_i}{\partial x_j}$ where $D$ is the following $d$-by-$d$ matrix:
\begin{align}\label{eqn:expression-for-matrix-D}
	D &\defeq A(x) \nabla^2 f(x) + \frac{\sqrt{\Gp} \nabla f(x) (\nabla g(x))^\top}{\sqrt{g(x)}} - \frac{1}{\sigma\saddleconst} \Phi'\Big(\frac{- \saddleconst I - \nabla^2 f(x)}{\sigma \saddleconst}\Big) T(x)[\nabla f(x)].
\end{align}
Note that $\dvg (A(x)\nabla f(x))$ is equal to the trace of matrix $D$. In order to proceed, we perform a case study on the value of $g(x)$.

\paragraph{Case $g(x) < \epsilon$:} We first upper bound the trace of $A(x)\nabla f^2(x)$, which can be written as:
\begin{align}\label{eqn:expand-weird-term-AF}
	A(x)\nabla f^2(x) = 2\sqrt{\Gp\,g(x)} \nabla f^2(x) + \Phi\Big(\frac{- \saddleconst I - \nabla^2 f(x)}{\sigma \saddleconst}\Big) \nabla f^2(x).
\end{align}
The trace of the first term on the right-hand side is bounded by $2\sqrt{\Gp\,g(x)}\Lp$. For the second term, we
assume that the matrix $\nabla^2 f(x)$ has eigenvalues $\lambda_1 \leq \lambda_2 \leq \dots \lambda_d$ with associated eigenvectors $u_1,\dots,u_d$. As a consequence, the matrix $\Phi(\frac{- \saddleconst I - \nabla^2 f(x)}{\sigma \saddleconst})$ has the same set of eigenvectors, but with eigenvalues $\Phi(\frac{-\lambda_1/\saddleconst - 1}{\sigma}), \dots, \Phi(\frac{-\lambda_d/\saddleconst - 1}{\sigma})$. Thus, the trace of this term is equal to
\[
	\tr\Big(\Phi\Big(\frac{- \saddleconst I - \nabla^2 f(x)}{\sigma \saddleconst}\Big) \nabla f^2(x)\Big) = \sum_{i=1}^d \lambda_i\, \Phi\Big(\frac{-\lambda_i/\saddleconst - 1}{\sigma}\Big).
\]
By the assumptions $x\in K\backslash U$ and $g(x)< \epsilon$, and using the definition of $\epsilon$-approximate local minima, we obtain $\lambda_1 \leq -\saddleconst$. As a consequence 
\[
\lambda_1\, \Phi\Big(\frac{-\lambda_1/\saddleconst - 1}{\sigma}\Big) \leq \lambda_1\, \Phi(0) \leq -\saddleconst/2.
\]
For other eigenvalues, if $\lambda_i$ is negative, then we use the upper bound $\lambda_i \Phi(\frac{-\lambda_i/\saddleconst - i}{\sigma}) < 0$; If $\lambda_i$ is positive, then we have $\lambda_i \Phi(\frac{-\lambda_i/\saddleconst - 1}{\sigma}) \leq \lambda_i \Phi(-\frac{1}{\sigma}) \leq \lambda_i \frac{\sqrt{\epsilon}}{4\Lp}$. Combining these relations, we have
\begin{align*}
	\tr\Big(\Phi\Big(\frac{- \saddleconst I - \nabla^2 f(x)}{\sigma \saddleconst}\Big) \nabla f^2(x)\Big) \leq - \frac{\saddleconst}{2} + \frac{\sqrt{\epsilon}}{4\Lp} \sum_{i=1}^d [\lambda_i]_+ \leq - \frac{\saddleconst}{2} + \frac{\sqrt{\epsilon}}{4\Lp} \lncs{\nabla^2 f(x)} \leq -\frac{\saddleconst}{4}.
\end{align*}
Combining this inquality with the upper bound on the first term of~\eqref{eqn:expand-weird-term-AF}, we obtain
\begin{align}\label{eqn:case-1-first-term-bound}
	\tr(A(x)\nabla f^2(x)) \leq 2\sqrt{\Gp\,g(x)}\Lp - \saddleconst/4.
\end{align}
Thus, we have upper bounded the trace of first term on the right-hand side of~\eqref{eqn:expression-for-matrix-D}.

For the second term on the right-hand side of~\eqref{eqn:expression-for-matrix-D} , we have
\begin{align}\label{eqn:case-1-second-term-bound}
	\tr\Big(\frac{\sqrt{\Gp}\nabla f(x) (\nabla g(x))^\top}{\sqrt{g(x)}}\Big) = \frac{\sqrt{\Gp}\inprod{\nabla f(x)}{\nabla g(x)}}{\sqrt{g(x)}} \leq \sqrt{\Gp\,g(x)}\, \ltwos{\nabla g(x)} \leq \sqrt{\Gp\,g(x)} \Lp,
\end{align}
where the last inequality uses the relation $\nabla g(x) = \frac{(\nabla^2 f(x))\nabla f(x)}{g(x)}$, so that $\ltwos{\nabla g(x)} \leq \ltwos{\nabla^2 f(x)} \leq \lncs{\nabla^2 f(x)} \leq \Lp$.

For the third term on the right-hand side of~\eqref{eqn:expression-for-matrix-D}, since $0 \preceq \Phi'(\frac{- \saddleconst I - \nabla^2 f(x)}{\sigma \saddleconst}) \preceq I$, we have
\[
	\tr(\mbox{the third term}) \leq \lncs{\mbox{the third term}} \leq \frac{1}{\sigma \saddleconst} \lncs{T(x)[\nabla f(x)]}.
\]
By Assumption~\ref{assumption:regularity-for-cheeger}, the function $f$ satisfies $\frac{\nucnorm{\nabla^2 f(x) - \nabla^2 f(y)}}{\ltwos{x-y}} \leq \Hp$, which implies $\nucnorm{T(x)[u]} \leq \Hp\ltwos{u}$ for any $x\in K$ and $u\in \R^d$. As a consequence, the term $\lncs{T(x)[\nabla f(x)]}$ is bounded by 
\[
\lncs{T(x)[\nabla f(x)]} \leq \Hp\ltwos{\nabla f(x)} = g(x) \Hp,
\]
which further implies
\begin{align}\label{eqn:case-1-third-term-bound}
	\tr(\mbox{the third term}) \leq \frac{g(x) \Hp}{\sigma \saddleconst}.
\end{align}
Combining upper bounds~\eqref{eqn:case-1-first-term-bound},~\eqref{eqn:case-1-second-term-bound},~\eqref{eqn:case-1-third-term-bound} implies
\[
	\dvg A(x)\nabla f(x) \leq 3\sqrt{\Gp\,g(x)}L + \frac{g(x) \Hp}{\sigma \saddleconst} - \saddleconst/4.
\]

\paragraph{Case $g(x) \geq \epsilon$:} The proof is similar to the previous case. For the first term on the right-hand side of equation~\eqref{eqn:expression-for-matrix-D}, we follow the same arguments for establishing the upper bound~\eqref{eqn:case-1-first-term-bound}, but without using the relation $\lambda_1\leq -\saddleconst$ (because conditioning on $g(x) \geq \epsilon$, the definition of approximate local minima won't give $\lambda_1\leq -\saddleconst$). Then the trace of $A(x)\nabla f^2(x)$ is bounded by:
\begin{align}\label{eqn:case-2-first-term-bound}
	\tr(A(x)\nabla f^2(x)) \leq 2\sqrt{\Gp\,g(x)} \Lp + \frac{\sqrt{\epsilon}}{4\Lp} \lncs{\nabla^2 f(x)} \leq 2\sqrt{g(x)} \Lp + \saddleconst/4.
\end{align}
For the second and the third term, the upper bounds \eqref{eqn:case-1-second-term-bound} and \eqref{eqn:case-1-third-term-bound} still hold, so that
\[
	\dvg A(x)\nabla f(x) \leq 3\sqrt{\Gp\,g(x)}L + \frac{g(x) \Hp}{\sigma \saddleconst} + \saddleconst/4.
\]

\noindent Combining the two cases completes the proof.


\section{Proof of Theorem~\ref{theorem:pop-general}}
\label{sec:proof-theorem-pop-general}

We apply the general Theorem~\ref{theorem:sld-general} to prove this theorem. In order to apply Theorem~\ref{theorem:sld-general}, the first step is to show that the function $\ftilde$ satisfies Assumption~\ref{assumption:conditions-for-szld}. Recall that the function is defined by:
\begin{align}\label{eqn:define-smoothing}
	\ftilde(x) = \E[f(x+z)] \quad \mbox{where} \quad z \sim N(0,\sigma^2 I),
\end{align}
and its stochastic gradient is computed by:
\begin{align*}
	g(x) \defeq \frac{z}{\sigma^2}(\ell(x+z;a)-\ell(x;a)).
\end{align*} 
By Assumption~\ref{assumption:erm}, the function $\ftilde$ is uniformly bounded in $[0,B]$. The following lemma captures additional properties of functions $\ftilde$ and $g$. See Appendix~\ref{sec:proof-lemma-property-of-smoothing} for the proof.

\begin{lemma}\label{lemma:property-of-smoothing}
The following properties hold:
\begin{enumerate}
\item For any $x\in K$, the stochastic function $g$ satisfies $\E[g(x)|x] = \nabla \ftilde(x)$. For any vector $u\in \R^d$ with $\ltwos{u} \leq \frac{\sigma}{2 B}$, it satisfies 
\[
\E[e^{\inprod{u}{g(x)}^2}|x] \leq e^{\ltwos{u}^2(2 B/\sigma)^2}.
\]
\item The function $\ftilde$ is $(2 B/\sigma^2)$-smooth.
\end{enumerate}
\end{lemma}

\noindent

Lemma~\ref{lemma:property-of-smoothing} shows that $\ftilde$ is an $L$-smooth function, with $L = (2 B/\sigma^2)$. In addition, the stochastic gradient $g$ satisfies the third condition of Assumption~\ref{assumption:conditions-for-szld} with $\bmax = \frac{\sigma}{2 B}$ and $\Gg = \frac{2 B}{\sigma}$. As a consequence, Theorem~\ref{theorem:sld-general} implies the risk bound:
\begin{align}\label{eqn:pop-bound-goal}
\ftilde(\xhat) \leq \sup_{x:\,d(x,U)\leq \rho} \ftilde(x), 
\end{align}
We claim the following inequality:
\begin{align}\label{eqn:ftilde-F-closeness}
	\mbox{If } \sigma = \frac{\nu}{\max\{ \Gp, B / \rhoK \}} \quad\mbox{then}\quad \sup_{x\in K} |\ftilde(x)- F(x)| \leq 2\nu.
\end{align}
We defer the proof of claim~\eqref{eqn:ftilde-F-closeness} to the end of this section, focusing on its consequence. Let $\sigma$ take the value in claim~\eqref{eqn:ftilde-F-closeness}. The conseuqence of~\eqref{eqn:ftilde-F-closeness} and the $\Gp$-Lipschitz continuity of the function $F$ imply:
\[
	F(\xhat) \leq \ftilde(\xhat) + 2\nu \leq \sup_{x:\,d(x,U)\leq \rho} \ftilde(x) + 2\nu \leq \sup_{x:\,d(x,U)\leq \rho} F(x) + 4\nu \leq \sup_{x\in U} F(x) + 4\nu + \Gp\rho,
\]
By choosing $\rho\defeq \nu/\Gp$, we establish the risk bound $F(\xhat) \leq \sup_{x\in U}F(x) +5\nu$. It remains to establish the iteration complexity bound.

According to Theorem~\ref{theorem:sld-general}, by choosing stepsize $\eta \defeq\eta_0$, SGLD achieves the risk bound~\eqref{eqn:pop-bound-goal} with iteration number polynomial in $(B, \Lf, \Gg, \log(1/\SmallProbParam), d, \xi, \hmax^{-1}, \bmax^{-1}, \rho^{-1}, \C^{-1}_{(\xi\ftilde)}(K\backslash U))$, where $(\Lf, \Gg, \bmax)$ depend on $\sigma$. Therefore, it remains to lower bound the restricted Cheeger constant $\C_{\xi\ftilde}(K\backslash U)$. By combining the claim~\eqref{eqn:ftilde-F-closeness} with inequality~\eqref{eqn:robustness-property}, we obtain
\[
	\C_{(\xi \ftilde)}(K\backslash U) \geq e^{-4\xi \nu} \C_{(\xi F)}(K\backslash U) \geq e^{-4} \C_{(\xi F)}(K\backslash U).
\]
It means that $\C_{(\xi \ftilde)}(K\backslash U)$ and $\C_{(\xi F)}(K\backslash U)$ differs by a constant multiplicative factor. Finally, plugging in the values of $\rho$ and $\sigma$ completes the proof. 

\paragraph{Proof of Claim~\eqref{eqn:ftilde-F-closeness}} We define an auxiliary function $\Ftilde$ as follow:
\begin{align}\label{eqn:define-F-smoothing}
	\Ftilde(x) = \E_z[F(x+z)] \quad \mbox{where} \quad z \sim N(0,\sigma^2 I),
\end{align}
Since $f(x) \in [F(x) - \nu, F(x) + \nu]$, the definitions~\eqref{eqn:define-smoothing} and~\eqref{eqn:define-F-smoothing} imply $\ftilde(x) \in [\Ftilde(x) - \nu, \Ftilde(x) + \nu]$. The $\Gp$-Lipschitz continuity of function $F$ implies that for any $x\in K$ and $y\in \Kbar$, there is $|F(y)-F(x)| \leq \Gp\ltwos{y-x}$. For any $x\in K$ and $y\notin \Kbar$, we have $F(x),F(y) \in  [0, B]$ and that the distance between $x$ and $y$ is at least $\rhoK$, thus $|F(y)-F(x)| \leq B \leq (B / \rhoK) \ltwos{y-x}$. As a consequence, for any $x\in K$ we have:
\begin{align*}
	|F(x+z) - F(x)| &\leq \indicator[x+z\in \Kbar]\,\Gp\ltwos{z} + \indicator[x+z\notin \Kbar ]\,(B / \rhoK) \ltwos{z} \leq \max\{ \Gp, B / \rhoK \}\, \ltwos{z}.
\end{align*}
Taking expectation over $z$ on both sides and using Jensen's inequality, we obtain
\[
	|\Ftilde(x) - F(x)| \leq \E[|F(x+z) - F(x)|] \leq \max\{ \Gp, B / \rhoK\} \E[\ltwos{z}] \leq \max\{ \Gp, B / \rhoK \}\,\sigma.
\]
Thus, by choosing $\sigma \defeq \frac{\nu}{\max\{ \Gp, B / \rhoK \}}$, it ensures that for any $x\in K$: 
\begin{align*}
	|\ftilde(x) - F(x)| \leq |\ftilde(x) - \Ftilde(x)| + |\Ftilde(x) - F(x)| \leq 2\nu.
\end{align*}


\subsection{Proof of Lemma~\ref{lemma:property-of-smoothing}}
\label{sec:proof-lemma-property-of-smoothing}

\noindent
(1) The function $\ftilde$ is a differentiable function, because it is the convolution of a bounded function $f$ and a Gaussian density function (which is infinite-order differentiable). We can write the gradient vector $\nabla \ftilde(x)$ as:
\begin{align*}
	\nabla \ftilde(x) = \frac{1}{(2\pi)^{d/2}}  \frac{\partial}{\partial x} \int e^{-\frac{\ltwos{z}^2}{2\sigma^2}} f(x+z) dz.
\end{align*}
Let $z'\defeq x+z$. By change of variables, the above equation implies:
\begin{align}
	\nabla \ftilde(x) &= \frac{1}{(2\pi)^{d/2}}  \frac{\partial}{\partial x} \int e^{-\frac{\ltwos{z'-x}^2}{2\sigma^2}} f(z') dz' = \frac{1}{(2\pi)^{d/2}}  \int \Big( \frac{\partial}{\partial x}  e^{-\frac{\ltwos{z'-x}^2}{2\sigma^2}}\Big) f(z') dz'.\nonumber\\
	&= \frac{1}{(2\pi)^{d/2}}  \int \frac{z'-x}{\sigma^2}\,e^{-\frac{\ltwos{x-z'}^2}{2\sigma^2}} f(z') dz' = \frac{1}{(2\pi)^{d/2}}  \int \frac{z}{\sigma^2} e^{-\frac{\ltwos{z}^2}{2\sigma^2}} f(x+z) dz\nonumber\\
	&= \E\Big[\frac{z}{\sigma^2}f(x+z)\Big] \stackrel{\textrm (i)}{=} \E\Big[\frac{z}{\sigma^2}(f(x+z)-f(x))\Big] \stackrel{\textrm (ii)}{=} \E\Big[\frac{z}{\sigma^2}(\ell(x+z;a)-\ell(x;a))\Big].\label{eqn:gradient-change-of-variable}
\end{align}
For the above deduction, equation (i) holds because $\E[zf(x)] = \E[z]\,\E[f(x)] = 0$; equation (ii) holds because $\E[\ell(y;a)|y] = f(y)$ for any $y\in K$. It shows that $g(x)$ is an unbiased estimate of $\nabla \ftilde(x)$. Since $\ell(\cdot;a) \in [0, B]$, any vector $u\in \R^d$ satisfies
\[
	(\inprod{u}{g(x)})^2 = (\ell(x+z;a)-\ell(x;a))^2\big(\langle u,\frac{z}{\sigma^2}\rangle\big)^2 \leq B^2 \big(\langle u,\frac{z}{\sigma^2}\rangle\big)^2 = \frac{B^2\ltwos{u}^2}{\sigma^2}\big(\inprod{\frac{u}{\ltwos{u}}}{\frac{z}{\sigma}}\big)^2
\]
Thus the following bound holds:
\[
	\E[e^{(\inprod{u}{g(x)})^2}|x] \leq \E[e^{\frac{B^2 \ltwos{u}^2}{\sigma^2}(\inprod{u/\ltwos{u}}{z/\sigma})^2}].
\]
Notice that $\inprod{u/\ltwos{u}}{z/\sigma}$ satisfies the standard normal distribution. Thus the right-hand side of the above inequality is bounded by 
\[
	\E[e^{\frac{B^2 \ltwos{u}^2}{\sigma^2}(\inprod{u/\ltwos{u}}{z/\sigma})^2}] = \frac{1}{\sqrt{1 - 2 B^2 \ltwos{u}^2/\sigma^2}} \leq e^{B^2 \ltwos{u}^2  (2/\sigma)^2} \quad \mbox{if} \quad \ltwos{u} \leq \frac{\sigma}{2 B}.
\]
Combining the two inequalities above completes the proof.\\

\noindent
(2) In order to bound the smoothness of the function $\ftilde(x)$, we derive the second derivative of $\ftilde(x)$ using equation~\eqref{eqn:gradient-change-of-variable}:
\begin{align}
\nabla^2 \ftilde(x) &= \frac{\partial}{\partial x} (\nabla \ftilde(x)) = \frac{1}{(2\pi)^{d/2}}  \frac{\partial}{\partial x}  \int \frac{z}{\sigma^2} e^{-\frac{\ltwos{z}^2}{2\sigma^2}} f(x+z) dz \nonumber\\
&= \frac{1}{(2\pi)^{d/2}}  \frac{\partial}{\partial x}  \int \frac{z'-x}{\sigma^2} e^{-\frac{\ltwos{z'-x}^2}{2\sigma^2}} f(z') dz' = \frac{1}{(2\pi)^{d/2}}   \int \Big(\frac{\partial}{\partial x}  \frac{z'-x}{\sigma^2} e^{-\frac{\ltwos{z'-x}^2}{2\sigma^2}}\Big) f(z') dz'\nonumber\\
&= \frac{1}{(2\pi)^{d/2}} \int \Big(\frac{(z'-x)(z'-x)^\top}{\sigma^4} - \frac{I}{\sigma^2}\Big) e^{-\frac{\ltwos{z'-x}^2}{2\sigma^2}} f(z') dz'
= \E\Big[ \frac{z z^\top - \sigma^2 I}{\sigma^4}f(x+z)\Big].
\label{eqn:hessian-change-of-variable}
\end{align}
Using the fact $f(x+z)\in [0,B]$, equation~\eqref{eqn:hessian-change-of-variable} implies:
\[
	\ltwos{\nabla^2 \ftilde(x)} \leq \frac{1}{\sigma^4} \ltwos{\E[z z^\top f(x+z)]} + \frac{1}{\sigma^2}\ltwos{\E[f(x+z) I]} \leq \frac{B}{\sigma^4}\ltwos{\E[z z^\top]} + \frac{B}{\sigma^2} = \frac{2 B}{\sigma^2},
\]
which establishes that the function $\ftilde$ is $(2 B/\sigma^2)$-smooth.

\section{Proof of Theorem~\ref{theorem:zero-one}}
\label{sec:proof-theorem-zero-one}

We use Theorem~\ref{theorem:pop-general} to upper bound the population risk as well as the time complexity. To apply the theorem, we need to verify Assumption~\ref{assumption:erm}. Recall that the parameter space is defined by $K\defeq\{x\in \R^d: 1/2\leq \ltwos{x}\leq 1\}$. Let $\Kbar\defeq\{x\in \R^d: 1/4\leq \ltwos{x}\leq 5/4\}$ be an auxiliary super-set.
The following lemma shows that these assumptions hold under our problem set-up.

\begin{lemma}\label{lemma:zero-one-assumptions}
The following properties hold:
\begin{enumerate}[(1)]
\item There exists $\hmax = \Omega(d^{-2})$ such that for any $x\in K$, $h\leq \hmax$ and $y\sim N(x,2h I)$, we have $P(y\in K) \geq 1/3$.
\item The function $F$ is 3-Lipschitz continuous in $\Kbar$.
\item For any $\nu,\SmallProbParam > 0$, if the sample size $n$ satisfies $n \gtrsim \frac{d}{\nu^2}$, then with probability at least $1-\SmallProbParam$ we have $\sup_{x\in \Kbar} |f(x)-F(x)| \leq \nu$. The notation ``$\lesssim$'' hides a poly-logarithmic function of $(d,1/\nu,1/\SmallProbParam)$.
\end{enumerate}
\end{lemma}

\noindent See Appendix~\ref{sec:proof-lemma-zero-one-assumptions} for the proof.

Let $\alpha_0\in (0,\pi/4]$ be an arbitrary angle. We define $U\subset K$ to be the set of points such that the angle between the point and $\xstar$ is bounded by $\alpha_0$, or equivalently:
\[
	U\defeq \{x\in K: \inprod{x/\ltwos{x}}{\xstar} \geq \cos(\alpha_0)\}.
\]
For any $x\in K$, the 3-Lipschitz continuity of function $F$ implies:
\begin{align*}
	F(x) = F(x/\ltwos{x}) \leq F(\xstar) + 3\norm{\frac{x}{\ltwos{x}}-\xstar}_2.
\end{align*}
By simple geometry, it is easy to see that
\[
	\norm{\frac{x}{\ltwos{x}}-\xstar}_2 = 2 \sin(\alpha/2) \leq 2 \sin(\alpha_0/2) \leq 2 \sin(\alpha_0).
\]
Thus, we have
\begin{align}\label{eqn:risk-bound-wrt-angle}
	F(x) \leq F(\xstar) + 6 \sin(\alpha_0).
\end{align}

Inequality~\eqref{eqn:risk-bound-wrt-angle} implies that for small enough $\alpha_0$, any point in $U$ is a nearly optimal solutions. Thus we can use $U$ as a target optimality set. The following lemma lower bounds the restricted Cheeger constant for the set $U$.

\begin{lemma}\label{lemma:zero-one-cheeger}
Assume that $d \geq 2$. For any $\alpha_0 \in (0, \pi/4]$, there are universal constant $c_1,c_2 > 0$ such that if we choose $\xi \geq \frac{c_1 d^{3/2}}{\BaseSignal \sin^2(\alpha_0)}$, then the restricted Cheeger constant is lower bounded by $\C_{(\xi F)}(K\backslash U) \geq c_2 d$.
\end{lemma}

\noindent See Appendix~\ref{sec:proof-zero-one-cheeger} for the proof.

Given a target optimality $\epsilon > 0$, we choose $\alpha_0 \defeq \arcsin(\epsilon/12)$. The risk bound~\eqref{eqn:risk-bound-wrt-angle} implies
\begin{align}\label{eqn:F-Fstar-epsilon}
	F(x) \leq F(\xstar) + \epsilon/2 \quad\mbox{for all}\quad x\in U.
\end{align}
Lemma~\ref{lemma:zero-one-assumptions} ensures that the pre-conditions of
Theorem~\ref{theorem:pop-general} hold with a small enough quantity $\nu$. Combining Theorem~\ref{theorem:pop-general} with inequality~\eqref{eqn:F-Fstar-epsilon}, with probability at least $1-\SmallProbParam$, SGLD achieves the risk bound:
\begin{align}\label{eqn:zero-one-risk-bound-0}
	F(\xhat) \leq \sup_{x\in U}F(x) + 5\nu \leq F(\xstar) + \epsilon/2 + 5\nu.
\end{align}
In order to have a small enough $\nu$, we want the functions $f$ and $F$ to be uniformly close. More precisely, we want the gap between them to satisfy:
\begin{align}\label{eqn:zero-one-uniform-gap}
\sup_{x\in \Kbar} |f(x)-F(x)|\leq \nu \defeq \min\Big\{ \frac{\BaseSignal\sin^2(\alpha_0)}{c_1 d^{3/2}}, \epsilon/10\Big\},
\end{align}
By Lemma~\ref{lemma:zero-one-assumptions}, this can be achieved by assuming a large enough sample size $n$. In particular, if the sample size satisfies $n \gtrsim \frac{d^4}{\BaseSignal^2\epsilon^4}$, then inequality~\eqref{eqn:zero-one-uniform-gap} is guaranteed to be true. The notation ``$\gtrsim$'' hides a poly-logarithmic function. 

If inequity~\eqref{eqn:zero-one-uniform-gap} holds, then $\nu \leq \epsilon/10$ holds, so that we can rewrite the risk bound~\eqref{eqn:zero-one-risk-bound-0} as $F(\xhat) \leq F(\xstar) + \epsilon$. By combining the choice of $\nu$ in \eqref{eqn:zero-one-uniform-gap} with the choice of $\xi \defeq \frac{c_1 d^{3/2}}{\BaseSignal \sin^2(\alpha_0)}$ in Lemma~\ref{lemma:zero-one-cheeger}, we find that the relation $\xi \in(0,1/\nu]$ hold, satisfying Theorem~\ref{theorem:pop-general}'s condition on $(\nu,\xi)$. As a result, Theorem~\ref{theorem:pop-general} implies that the iteration complexity of SGLD is bounded by the restricted Cheeger constant $\C_{(\xi F)}(K\backslash U)$. By Lemma~\ref{lemma:zero-one-cheeger}, the
restricted Cheeger constant is lower bounded by $\Omega(d)$, so that the iteration complexity is polynomial in $(d,1/\BaseSignal, 1/\epsilon, \log(1/\SmallProbParam))$.

\subsection{Proof of Lemma~\ref{lemma:zero-one-assumptions}}
\label{sec:proof-lemma-zero-one-assumptions}

(1) Let $x\in K$ be an arbitrary point and let $z\sim N(0,2h I)$. An equivalent way to express the relation $x+z\in K$ is the following sandwich inequality:
\begin{align}\label{eqn:equiv-condition-for-xz-in-K}
	1/4 - \ltwos{x}^2 - \ltwos{z}^2 \leq 2\inprod{x}{z} \leq 1 - \ltwos{x}^2 - \ltwos{z}^2.
\end{align}
For any $t > 0$, we consider a sufficient condition for inequality~\eqref{eqn:equiv-condition-for-xz-in-K}:
\begin{align*}
	\ltwos{z}^2 \leq 2 t h d \quad\mbox{and} \quad 2\inprod{x}{z} \in I_x \defeq
	[1/4 - \ltwos{x}^2, 1 - \ltwos{x}^2 - 2 t h d].
\end{align*}
The random variable $\frac{\ltwos{z}^2}{2h}$ satisfies a chi-square distribution with $d$ degrees of freedom. 
By Lemma~\ref{lemma:chi-square-concentration}, for any $t\geq 5$, the condition $\ltwos{z}^2 \leq 2t h d$ holds with probability at least $1- e^{-\Omega(t d)}$. 

Suppose that $t$ is a fixed constant, and $h$ is chosen to be $h\defeq \frac{c^2}{2 t d^2}$ for a constant $c>0$. Then the random variable $w_x \defeq 2\inprod{x}{z}$ satisfies a normal distribution $N(0; \frac{4\ltwos{x}^2(c/d)^2}{t} )$. The interval $I_x$, no matter how $x\in K$ is chosen, covers either $[-1/4,-(c/d)^2]$ or $[(c/d)^2,1/4]$. For $c\to 0$, we have $(c/d)^2 \ll \frac{2\ltwos{x}}{t}(c/d) \ll 1/4$, so that the probability of $w_x\in I_x$ is asymptotically lower bounded by $0.5$. It implies that there is a strictly positive constant $c$ (depending on the value of $t$) such that $P(w_x\in I_x) \geq 0.4$ for all $x\in K$. With this choice of $c$, we apply the union bound:
\[
	P(x+z\in K) \geq P(w_x\in I_x) - P(\ltwos{z}^2 > 2t h d) \geq 0.4 - e^{-\Omega(t d)}.
\]
By choosing $t$ to be a large enough constant, the above probability is lower bounded by $1/3$.\\

\noindent
(2) For two vectors $x, y \in \Kbar$, the loss values $\ell(x;a)$ and $\ell(y;a)$ are non-equal only when $\sgn(\inprod{x}{a}) \neq \sgn(\inprod{y}{a})$. Thus, we have the upper bound
\begin{align}\label{eqn:Fx-Fy-gap-bound}
	|F(x) - F(y)| \leq P\Big(\sgn(\inprod{x}{a}) \neq \sgn(\inprod{y}{a})\Big).
\end{align}
If we change the distribution of $a$ from uniform distribution to a normal distribution $N(0, I_{d\times d})$, the right-hand side of inequality~\eqref{eqn:Fx-Fy-gap-bound} won't change. Both $\inprod{x}{a}$ and $\inprod{x}{b}$ become normal random variables with correlation coefficient $\frac{\inprod{x}{y}}{\ltwos{x}\ltwos{y}}$. Under this setting, \citet{tong2012multivariate} proved that the right-side is equal to
\[
	P\Big(\sgn(\inprod{x}{a}) \neq \sgn(\inprod{y}{a})\Big)  = \frac{1}{\pi} \arccos\Big( \frac{\inprod{x}{y}}{\ltwos{x}\ltwos{y}} \Big)
\]
By simple algebra and using the fact that $\ltwos{x},\ltwos{y} \geq 1/4$, we have
\[
	\frac{\inprod{x}{y}}{\ltwos{x}\ltwos{y}} = \frac{1}{2}\Big(\frac{\ltwos{x}}{\ltwos{y}} + \frac{\ltwos{y}}{\ltwos{x}} - \frac{\ltwos{x-y}^2}{\ltwos{x}\ltwos{y}}\Big) \geq 1 - 8\ltwos{x-y}^2,
\]
Combining the above relations, and using the fact that $\arccos(t) \leq 3 \sqrt{1-t}$ for any $t\in [-1,1]$, we obtain
\[
	|F(x) - F(y)| = \frac{1}{\pi} \arccos\Big( \frac{\inprod{x}{y}}{\ltwos{x}\ltwos{y}} \Big) \leq \frac{3 \sqrt{8}}{\pi} \ltwos{x-y} \leq 3\,\ltwos{x-y},
\]
which shows that the function $F$ is 3-Lipschitz continuous in $\Kbar$.\\

\noindent
(3) Since function $f$ is the empirical risk of a linear classifier, its uniform convergence rate can be characterized by the VC-dimension. The VC-dimension of linear classifiers in a $d$-dimensional space is equal to $d+1$. Thus, the concentration inequality of~\citet{vapnik1999overview} implies that with probability at least $1-\SmallProbParam$, we have
\[
	\sup_{x\in \R^d} |f(x)-F(x)| \leq U(n) \defeq c\,\sqrt{\frac{d(\log(n/d)+1)+\log(1/\SmallProbParam)}{n}},
\]
where $c > 0$ is a universal constant. When $n\geq d$, the upper bound $U(n)$ is a monotonically decreasing function of $n$. In order to guarantee $U(n) \leq \nu$, it suffices to choose $n \geq n_0$ where the number $n_0$ satisfies:
\begin{align}\label{eqn:n0-constraint}
 	n_0  = \max\{n\in \R:~U(n)= \nu\}.
\end{align}
It is easy to see that $n_0$ is polynomial in $(d, 1/\nu, 1/\SmallProbParam)$. Thus, by the definition of $U(n)$, we have
\[
	\nu = c\,\sqrt{\frac{d(\log(n_0/d)+1)+\log(1/\SmallProbParam)}{n_0}}
	\leq c\,\sqrt{\frac{d\cdot\mbox{polylog}(d, 1/\nu, 1/\SmallProbParam )}{n_0}}.
\]
It implies $n_0 \lesssim \frac{d}{\nu^2}$, thus completes the proof.


\subsection{Proof of Lemma~\ref{lemma:zero-one-cheeger}}
\label{sec:proof-zero-one-cheeger}

Note that the population risk $F$ can be non-differentiable. In order to apply Lemma~\ref{lemma:vector-field} to lower bound the restricted Cheeger constant, we define a smoothed approximation of $F$, and apply Lemma~\ref{lemma:vector-field} on the smoothed approximation.
For an arbitrary $\sigma > 0$, we define $\Ftilde$ to be a smoothed approximation of the population risk:
\[
	\Ftilde(x) \defeq \E[F(x+z)] \quad\mbox{where}\quad z \sim N(0,\sigma^2 I_{d\times d}).
\]
By Lemma~\ref{lemma:zero-one-assumptions}, the function $F$ is 3-Lipschitz continuous, so that $\Ftilde$ uniformly converges to $F$ as $\sigma\to 0$. It means that
\[
	\lim_{\sigma\to 0} \C_{(\xi \Ftilde)}(K\backslash U) = \C_{(\xi F)}(K\backslash U).
\]
It suffices to lower bound $\C_{(\xi \Ftilde)}(K\backslash U)$ and then take the limit $\sigma\to 0$. The function $\Ftilde$ is continuously differentiable, so that we can use Lemma~\ref{lemma:vector-field} to lower bound $\C_{(\xi \Ftilde)}(K\backslash U)$. 

Consider an arbitrary constant $ 0 < t \leq 1/6$. We choose a small enough $\sigma > 0$ such that for $z\sim (0,\sigma^2 I)$, the inequality $\E[\ltwos{z}] \leq t$ holds, and the event $\event_t \defeq \{\ltwos{z} \leq t\}$ holds with probability at least $1/2$. It is clear that the choice of $\sigma$ depends on that of $t$, and as $t\to 0$, we must have $\sigma \to 0$.

The first step is to define a vector field that satisfies the conditions of Lemma~\ref{lemma:vector-field}. For arbitrary $\delta \in [0,1]$, we define a vector field $\phi_\delta$ such that:
\begin{align}
	\phi_\delta(x) \defeq \frac{1}{3}(\inprod{x}{\xstar}\,x - \ltwos{x}^2\, \xstar) + \frac{\delta}{3}(\ltwos{x}^2- 5/8)\,x,
\end{align}
and make the following claim.

\begin{claim}\label{claim:zero-one-valid-phi}
For any $\delta \in (0,1]$, we can find a constant $\SmallVar_0 > 0$ such that $\ltwos{\phi_\delta(x)} \leq 1$ and $x-\SmallVar\phi_\delta(x) \in K$ holds for arbitrary $x\in K$ and $\SmallVar \in (0, \SmallVar_0]$.
\end{claim}

The claim shows that $\phi_\delta$ satisfies the conditions of Lemma~\ref{lemma:vector-field} for any $\delta\in (0,1]$, so that given a scalar $\xi > 0$, the lemma implies
\[
	\C_{\xi \Ftilde(x)}(K\backslash U) \geq \inf_{x\in K\backslash U} \Big\{ \xi\, \inprod{\phi_\delta(x)}{\nabla \Ftilde(x)} - \dvg \phi_\delta(x) \Big\}.
\]
The right-hand side is uniformly continuous in $\delta$, so that if we take the limit $\delta \to 0$, we obtain the lower bound:
\begin{align}\label{eqn:zero-one-cheeger-first-step}
	\C_{\xi \Ftilde(x)}(K\backslash U) \geq \inf_{x\in K\backslash U} \Big\{ \xi\, \inprod{\phi_0(x)}{\nabla \Ftilde(x)} - \dvg \phi_0(x) \Big\}.
\end{align}

It remains to lower bound the right-hand side of inequality~\eqref{eqn:zero-one-cheeger-first-step}. Recall that $\Ftilde(x) = \E[F(x+z)]$. The definition of the gradient of $\Ftilde$ implies
\begin{align*}
	\inprod{\phi_0(x)}{\nabla \Ftilde(x)} &= \lim_{\SmallVar\to 0} \frac{\Ftilde(x + \SmallVar \phi_0(x)) - \Ftilde(x)}{\SmallVar}
	= \lim_{\SmallVar\to 0} \E\Big[\frac{F(x + z + \SmallVar \phi_0(x)) - F(x+z)}{\SmallVar}\Big]
\end{align*}
For the right-hand side, we prove lower bound for it using the Massart noise model. We start by simplifying the fraction term inside the expectation. Without loss of generality, assume that $\SmallVar \in (0,0.2]$, then the definition of $\phi_0$ implies:
\begin{align*}
F(x + z + \SmallVar \phi_0(x)) &= F\Big((1+\frac{\SmallVar}{3}\inprod{x}{\xstar})x + z - \frac{\SmallVar}{3}\ltwos{x}^2\,\xstar\Big)\\
 & \stackrel{\textrm (i)}{=} F\Big(x + \frac{z}{1+ \frac{\SmallVar}{3}\inprod{x}{\xstar}} - \frac{\frac{\SmallVar}{3} \ltwos{x}^2 \xstar}{1 + \frac{\SmallVar}{3}\inprod{x}{\xstar}}\Big)\\
&\stackrel{\textrm (ii)}{\geq} F(x+z-\frac{\SmallVar\ltwos{x}^2}{3}\xstar) - 2\SmallVar\ltwos{z} - \frac{2\SmallVar^2}{3},
\end{align*}
where equation (i) uses $F(x) = F(\alpha x)$ for any $\alpha > 0$. To derive inequality (ii), we used the fact that
$1/(1 + \frac{\SmallVar}{3}\inprod{x}{\xstar}) \in [1-\frac{2\SmallVar}{3}, 1+\frac{2\SmallVar}{3}]$ for any $x\in K$, $\SmallVar\in(0,0.2]$, and the property that $F$ is 3-Lipschitz continuous. Combining the two equations above, and using the fact $\E[\ltwos{z}] \leq t$, we obtain:
\begin{align}
	\inprod{\phi_0(x)}{\nabla \Ftilde(x)} 
	&\geq \lim_{\SmallVar\to 0} \E\Big[\frac{F(x + z - \frac{\SmallVar\ltwos{x}^2}{3} \xstar) - F(x+z)}{\SmallVar}\Big] - 2t\nonumber\\
	&= \lim_{\SmallVar\to 0}\E\Big[\frac{F(\frac{3(x + z)}{\ltwos{x}^2} - \SmallVar\xstar) - F(\frac{3(x + z)}{\ltwos{x}^2})}{\SmallVar} \Big] - 2t.\label{eqn:zero-one-cheeger-limit}
\end{align}
We further simplify the lower bound~\eqref{eqn:zero-one-cheeger-limit} by the following claim, which is proved using properties of the Massart noise.

\begin{claim}\label{claim:improve-loss-ratio}
For any $x\in \R^d$ and any $\SmallVar > 0$, we have
$F(x - \SmallVar\xstar) - F(x) \geq 0$. Moreover, for any $x\in \R^d: \ltwos{x} \geq 1$, let $\alpha$ be the angle between $x$ and $\xstar$, then we have: 
\[
	\frac{F(x - \SmallVar\xstar) - F(x)}{\SmallVar}
	\geq  \frac{3\BaseSignal |\sin(\alpha)|}{5\pi\ltwos{x}}  \Big(\frac{ |\sin(\alpha)|\sqrt{1-\SmallVar^2} }{2\sqrt{d}} - \SmallVar \Big).
\]
\end{claim}

When the event $\event_t$ holds, we have $\ltwos{z}\leq t\leq 1/6 \leq \ltwos{x}/3$, so that $\ltwos{\frac{3(x + z)}{\ltwos{x}^2}} \in [2,4]$. Combining with inequality~\eqref{eqn:zero-one-cheeger-limit} and Claim~\ref{claim:improve-loss-ratio}, we have
\begin{align}
	\inprod{\phi_0(x)}{\nabla \Ftilde(x)} &\geq P(\event_t)\,\lim_{\SmallVar\to 0}\E\Big[\frac{F(\frac{3(x + z)}{\ltwos{x}^2} - \SmallVar\xstar) - F(\frac{3(x + z)}{\ltwos{x}^2})}{\SmallVar} \mid \event_t \Big] - 2t\nonumber\\
	&\geq \frac{1}{2}\times \lim_{\SmallVar\to 0} \E\Big[\frac{3\BaseSignal |\sin(\alpha_{x+z})|}{5\pi\times 4}  \Big(\frac{ |\sin(\alpha_{x+z})|\sqrt{1-\SmallVar^2} }{2\sqrt{d}} - \SmallVar \Big)\mid \event_t\Big] - 2t\nonumber\\
	&= \frac{3\BaseSignal \E[\sin^2(\alpha_{x+z})\mid \event_t]}{80\pi\sqrt{d}},
	\label{eqn:apply-lemma-loss-ratio}
\end{align}
where $\alpha_{x+z}$ represents the angle between $x+z$ and $\xstar$.

In order to lower bound the divergence term $\dvg \phi_0(x)$, let $H(x)\in \R^{d\times d}$ be the Jacobian matrix of $\phi_0$ at point $x$ (i.e.~$H_{ij} \defeq \frac{\partial (\phi_0(x))_i}{\partial x_j}$), then we have
\begin{align*}
	 H(x) = \frac{1}{3}\Big(\inprod{x}{\xstar} I + x (\xstar)^\top - 2 \xstar x^\top \Big).
\end{align*}
It means that $\dvg \phi_0(x) = \tr(H(x)) = \frac{d-1}{3}\inprod{x}{\xstar} = \frac{(d-1)\ltwos{x}}{3} \cos(\alpha_x)$. Combining this equation with inequalities~\eqref{eqn:zero-one-cheeger-first-step},~\eqref{eqn:apply-lemma-loss-ratio}, and taking the limits $t\to 0$, $\sigma\to 0$, we obtain:
\begin{align}\label{eqn:zero-one-cheeger-secon-step}
	\C_{\xi F(x)}(K\backslash U) \geq \inf_{x\in K\backslash U} \Big\{\underbrace{\frac{3\xi \BaseSignal \sin^2(\alpha_x) }{80\pi\sqrt{d}} - \frac{(d-1)\ltwos{x}}{3}\cos(\alpha_x)}_{\defeq L(x)}\Big\}.
\end{align}
where $\alpha_x$ represents the angle between $x$ and $\xstar$.

According to inequality~\eqref{eqn:zero-one-cheeger-secon-step}, for any $\alpha_x\in (\pi-\alpha_0,\pi]$, we have
\[
	L(x) \geq - \frac{(d-1)\ltwos{x}}{3}\cos(\alpha_x) \geq \frac{(d-1)\ltwos{x}}{3}\cos(\alpha_0) \geq \frac{(d-1)}{6\sqrt{2}},
\]
where the last inequality follows since $\ltwos{x} \geq 1/2$ and $\alpha_0 \in (0,\pi/4]$. Otherwise, if $\alpha_x \in [\alpha_0, \pi - \alpha_0]$, then we have
\[
	L(x)\geq \frac{3\xi \BaseSignal \sin^2(\alpha_0) }{80\pi\sqrt{d}} - \frac{d-1}{3}.
\]
Once we choose $\xi \geq \frac{160\pi}{3}\frac{d^{3/2}}{\BaseSignal\sin^2(\alpha_0)}$, the above expression will be lower bounded by $d/3$, which completes the proof.

\paragraph{Proof of Claim~\ref{claim:zero-one-valid-phi}}

Since $\ltwos{x}\leq 1$ for any $x\in K$, it is easy to verify that $\ltwos{\phi_\delta(x)}\leq 1$. In order to verify $x - \SmallVar \phi_\delta(x) \in K$, we notice that
\[
\ltwos{x - \SmallVar \phi_\delta(x)}^2 = \ltwos{x}^2 + \SmallVar^2 \ltwos{\phi_\delta(x)}^2 - \frac{2\SmallVar\delta}{3} (\ltwos{x}^2 - 5/8) \ltwos{x}^2.
\]
As a consequence, we have
\[
	\Big|\ltwos{x - \SmallVar \phi_\delta(x)}^2 - 5/8\Big| \leq \Big|\ltwos{x}^2 - 5/8\Big|\, \Big(1 - \frac{2\SmallVar\delta}{3} \ltwos{x}^2 \Big) + \SmallVar^2
\]
The right-hand side will be maximized if $\ltwos{x}^2 = 1/4$. Thus, if we assume $\delta > 0$, then for any $\SmallVar < \delta/16$, it is easy to verify that the right-hand side is bounded by $3/8$. As a consequence, we have $\ltwos{x - \SmallVar \phi_\delta(x)}^2\in [1/4,1]$, which verifies that $x - \SmallVar \phi_\delta(x) \in K$.

\paragraph{Proof of Claim~\ref{claim:improve-loss-ratio}}

When $x = 0$ or $x - \SmallVar\xstar = 0$, it is easy to verify that $F(x - \SmallVar\xstar) - F(x) \geq 0$ by the definition of the loss function. Otherwise, we assume that $x\neq 0$ and $x - \SmallVar\xstar \neq 0$. In these cases, the events $\inprod{x}{a} \neq 0$ and $\inprod{x - \SmallVar\xstar}{a} \neq 0$ hold almost surely, so that we can assume that the loss function always takes zero-one values.

When the parameter changes from $x$ to $x - \SmallVar \xstar$, the value of $\ell(x;a)$ and $\ell(x-\SmallVar \xstar;a)$ are non-equal if and only if $\sgn(\inprod{x}{a}) \neq \sgn(\inprod{x-\SmallVar\xstar}{a})$. This  condition is equivalent of
\begin{align}\label{eqn:change-predict-condition}
\sgn(\inprod{x}{a}) = \sgn(\inprod{\xstar}{a}) \quad\mbox{and}\quad 
|\inprod{x}{a}| < \SmallVar |\inprod{\xstar}{a}|.
\end{align}
Let $\event$ be the event that condition~\eqref{eqn:change-predict-condition} holds. Under this event, when the parameter changes from $x$ to $x - \SmallVar \xstar$, the loss changes from $\frac{1-\Signal(a)}{2}$ to $\frac{1+\Signal(a)}{2}$. It means that the loss is non-decreasing with respect to the change $x \to x-\SmallVar\xstar$, and as a consequence, we have
$F(x - \SmallVar\xstar) - F(x) \geq 0$.

In order to establish the lower bound in Claim~\ref{claim:improve-loss-ratio}, we first lower bound the probability of event $\event$. In the proof of Lemma~\ref{lemma:zero-one-assumptions}. we have shown that this probability is equal to:
\[
	P(\event) = \frac{1}{\pi} \arccos\Big(\frac{\inprod{x}{x-\SmallVar \xstar}}{\ltwos{x}\ltwos{x+\SmallVar\xstar}}\Big)
\]
Let $\beta$ be the angle between $x$ and $x-\SmallVar\xstar$, then the right-hand side is equal to $\beta/\pi$. Using the geometric property that $\frac{\ltwos{x}}{|\sin (\alpha)|} = \frac{\SmallVar \ltwos{\xstar}}{|\sin(\beta)|}$, we have
\begin{align}\label{eqn:beta-angle-lower-bound}
	P(\event) = \frac{\beta}{\pi} \geq \frac{|\sin \beta|}{\pi} = \frac{\SmallVar |\sin(\alpha)|}{\ltwos{x} \pi} \geq \frac{\SmallVar |\sin(\alpha)|}{\ltwos{x} \pi}.
\end{align}

Conditioning on the event $\event$, when the parameter moves $x \to x-\SmallVar\xstar$, 
the loss $\ell(x;a)$ changes by amount $\Signal(a)$. Since the Massart noise forces $\Signal(a) \geq \BaseSignal |\inprod{\xstar}{a}|$, we can lower bound the gap $F(x - \SmallVar \xstar) - F(x)$ by lower bounding the expectation $\E[|\inprod{\xstar}{a}| \,\big|\, \event]$. We decompose the vector $a$ into two components: the component $a_1$ that is parallel to $x$ and the component $a_2$ that is orthogonal to $x$. Similarly, we can decompose the vector $\xstar$ into two components $\xstar_1$ and $\xstar_2$, parallel to and orthogonal to the vector $x$ respectively. The decomposition implies 
\[
	\inprod{\xstar}{a} = \inprod{\xstar_1}{a_1} + \inprod{\xstar_2}{a_2}.
\] 

For the first term on the right-hand side, we have $|\inprod{\xstar_1}{a_1}| \leq \ltwos{a_1} =  \frac{|\inprod{x}{a}|}{\ltwos{x}} \leq \frac{\SmallVar}{\ltwos{x}} \leq \SmallVar$ by condition~\eqref{eqn:change-predict-condition} and the assumption that $\ltwos{x} \geq 1$. For the second term, if we condition on $a_1$, then the vector $a_2$ is uniformly sampled from a $(d-1)$-dimensional sphere of radius $\sqrt{1-\ltwos{a_1}^2}$ that centers at the origin. The vector $\xstar_2$, constructed to be orthogonal to $x$, also belongs to the same $(d-1)$-dimensional subspace. Under this setting, \citet[][Lemma 4]{awasthi2015efficient} proved that
\[
	P\Big(|\inprod{\xstar_2}{a_2}| > \frac{\ltwos{\xstar_2}\sqrt{1-\ltwos{a_1}^2}}{2\sqrt{d}} \mid \event,a_1 \Big) \geq 1 - \sqrt{\frac{1}{2\pi}} \geq 3/5. 
\]
Using the bound $\ltwos{a_1} =  \frac{|\inprod{x}{a}|}{\ltwos{x}} \leq \frac{\SmallVar}{\ltwos{x}} \leq \SmallVar$, we marginalize  $a_1$ to obtain
\[
	P\Big(|\inprod{\xstar_2}{a_2}| > \frac{\ltwos{\xstar_2}\sqrt{1-\SmallVar^2}}{2\sqrt{d}} \mid \event \Big) \geq 3/5. 
\]
Recall that $\inprod{\xstar}{a} = \inprod{\xstar_1}{a_1} + \inprod{\xstar_2}{a_2}$ and $|\inprod{\xstar_1}{a_1}| \leq \SmallVar$. These two relations imply $|\inprod{\xstar}{a}| \geq |\inprod{\xstar_2}{a_2}| - |\inprod{\xstar_1}{a_1}| \geq |\inprod{\xstar_2}{a_2}| - \SmallVar$. Combining it with the relation $\ltwos{\xstar_2} = |\sin(\alpha)|$, we obtain
\begin{align*}
	P\Big(|\inprod{\xstar}{a}| \geq \frac{|\sin(\alpha)|\sqrt{1-\SmallVar^2}}{2\sqrt{d}} - \SmallVar \mid \event \Big) \geq 3/5. 
\end{align*}
As a consequence, we have
\begin{align}\label{eqn:not-in-weak-band-probability}
\E[|\inprod{\xstar}{a}| \,\big|\, \event] \geq \frac{3}{5} \Big(\frac{ |\sin(\alpha)|\sqrt{1-\SmallVar^2} }{2\sqrt{d}} - \SmallVar \Big).
\end{align}
Combining inequalities~\eqref{eqn:beta-angle-lower-bound},~\eqref{eqn:not-in-weak-band-probability} with the relation that $\Signal(a) \geq \BaseSignal |\inprod{\xstar}{a}|$, we obtain
\begin{align*}
	 F(x - \SmallVar \xstar) - F(x) & \geq P(\event)\times \BaseSignal\,\E[|\inprod{\xstar}{a}| \,\big|\, \event]\geq \frac{\SmallVar|\sin(\alpha)|}{\ltwos{x} \pi} \times \BaseSignal\,\frac{3}{5} \Big(\frac{ |\sin(\alpha)|\sqrt{1-\SmallVar^2} }{2\sqrt{d}} - \SmallVar \Big).
\end{align*}
which completes the proof.

\end{document}